\documentclass[10pt, conference]{ IEEEtran}
\IEEEoverridecommandlockouts
\AtBeginDocument{%
  \providecommand\BibTeX{{%
    Bib\TeX}}}
\usepackage{cite}
\usepackage{amsmath,amssymb,amsfonts}
\usepackage{algorithmic}
\usepackage{graphicx}
\usepackage{textcomp}
\usepackage{xcolor}
\def\BibTeX{{\rm B\kern-.05em{\sc i\kern-.025em b}\kern-.08em
    T\kern-.1667em\lower.7ex\hbox{E}\kern-.125emX}}
\usepackage{float}  

\usepackage{amsthm}
\usepackage{graphicx}
\usepackage{subcaption}
\usepackage{algorithmic}
\usepackage[ruled, linesnumbered]{algorithm2e}
\usepackage{textcomp}
\usepackage{xcolor}
\usepackage{multirow}
\usepackage{enumitem}
\usepackage{subcaption}
\usepackage[short]{optional}
\usepackage{comment}
\usepackage{balance}
\usepackage{booktabs}
\usepackage{hyperref}

\newtheorem{theorem}{Theorem}

\def\BibTeX{{\rm B\kern-.05em{\sc i\kern-.025em b}\kern-.08em
    T\kern-.1667em\lower.7ex\hbox{E}\kern-.125emX}}
\begin{document}


\title{Decentralized Rank Scheduling for Energy-Constrained Multi-Task Federated Fine-Tuning in Edge-Assisted IoV Networks}

\author{
	Bokeng~Zheng,
    Jianqiang~Zhong,
    Jiayi~Liu,
    Lei~Xue,~\IEEEmembership{Member,~IEEE}\\
    Xu~Chen,~\IEEEmembership{Senior Member,~IEEE,}
        Xiaoxi~Zhang,~\IEEEmembership{Member,~IEEE}
		
\thanks{B. Zheng, J. Zhong, J. Liu, X. Chen and X. Zhang are with the School of Computer Science and Engineering, Sun Yat-sen University, Guangzhou 510006, China~(E-mail: \{zhengbk6, zhongjq28, liujy373\}@mail2.sysu.edu.cn; \{chenxu35, zhangxx89\}@mail.sysu.edu.cn).}
\thanks{Lei Xue is with the School of Cyber Science and Technology, Sun Yat-Sen University, Shenzhen, China~(E-mail: qqxuelei@gmail.com).
}
\thanks{Xiaoxi Zhang is the corresponding author.}
}

\IEEEpubidadjcol
    \IEEEoverridecommandlockouts

\maketitle

\begin{abstract}
Large-scale Internet of Vehicles (IoV) deployments increasingly demand the on-device adaptation of foundation models to support diverse, mission-critical perception tasks. While federated fine-tuning offers a promising solution for efficient model specialization, existing approaches often struggle to reconcile the inherent conflict between stringent global energy budgets, heterogeneous task demands, and the high volatility of vehicular network connectivity. In this work, we introduce a hierarchical, adaptive framework that decouples multi-task fine-tuning into two interdependent optimization phases. First, we implement a feedback-loop mechanism at the infrastructure level that dynamically redistributes global energy budgets across concurrent tasks based on real-time convergence dynamics and resource utilization. Second, at the vehicle level, we formulate intra-task rank selection as an energy-constrained online learning problem, solved via a novel primal–dual bandit algorithm, UCB-DUAL, which provides theoretical guarantees on sublinear regret. Our approach effectively internalizes global energy constraints into local decision-making, allowing vehicles to autonomously navigate the complex trade-off between model accuracy, latency, and power consumption. Empirical evaluations using a large-scale IoV simulator, driven by real-world trajectory data, confirm that our proposed method significantly outperforms current federated fine-tuning baselines, offering a robust and scalable solution for resource-constrained vehicular intelligence.
\end{abstract}



\section{Introduction}
\label{sec:intro}

With the rapid development of smart cities, the Internet of Vehicles (IoV) has emerged as a critical infrastructure for enabling intelligent edge services~\cite{wang2021green}, including traffic perception, environmental monitoring, and autonomous driving~\cite{javaid2018smart,rath2018smart,ding2021overview}. These services are typically deployed in resource-constrained edge environments and must satisfy stringent requirements on latency, reliability, and energy efficiency~\cite{ding2021overview}. Meanwhile, vehicles continuously generate large volumes of heterogeneous data that reflect highly dynamic traffic and environmental conditions. Effectively exploiting such local data to adapt learning models to evolving edge scenarios has become a central challenge in IoV systems~\cite{Tomtit2024}. In this context, federated fine-tuning has become a practical paradigm, enabling vehicles to locally adapt pre-trained foundation models via lightweight updates rather than transmitting raw data or retraining full models~\cite{bian2025survey}. While this approach preserves privacy and reduces communication overhead, standard full-parameter fine-tuning remains prohibitive in IoV due to its excessive computational cost and energy consumption on vehicular devices.

\begin{figure}
\centering
\includegraphics[width=1\linewidth]{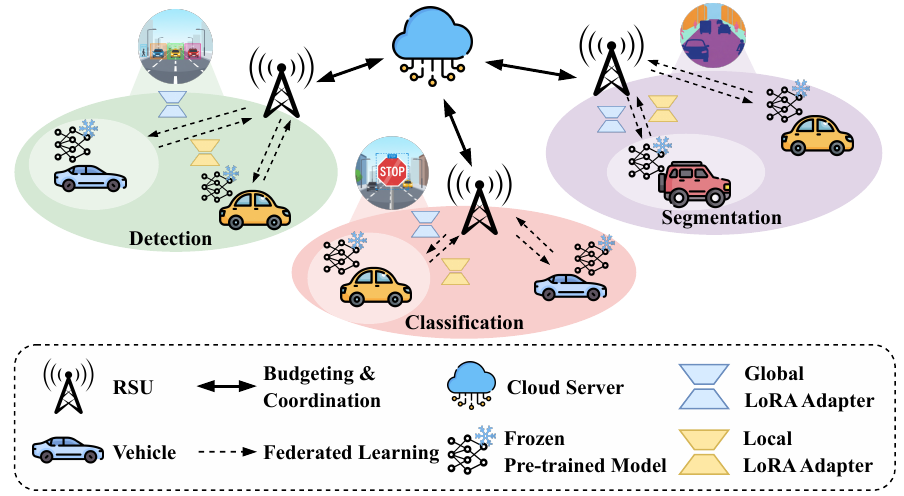}
\caption{Multi-task federated fine-tuning with IoV.}
\label{fig:intro}
\end{figure}

To overcome this limitation, parameter-efficient fine-tuning 
\\(PEFT) techniques~\cite{han2024parameter}, particularly Low-Rank Adaptation (LoRA)\\~\cite{hu2021lora}, provide an effective alternative to full fine-tuning. By injecting low-rank trainable modules into pre-trained models, LoRA achieves comparable performance with far fewer trainable parameters. This makes LoRA well suited for the edge-assisted IoV settings depicted in Fig.~\ref{fig:intro}, enabling lightweight on-vehicle updates under latency and energy constraints, and efficient aggregation at RSUs.

Despite these advantages, the choice of LoRA rank plays a critical role in determining system performance.
As shown in Fig.~\ref{fig:rank_com}, increasing rank consistently improves accuracy and convergence speed; however, this incurs higher latency and energy costs, revealing a critical accuracy–latency–energy trade-off.
Compared with full-parameter fine-tuning, LoRA-based adaptation achieves comparable or even superior performance with significantly lower system overhead. \textit{The reason is that low-rank parameterization effectively confines model updates to task-relevant subspaces while preserving the general knowledge of the pre-trained backbone, thereby providing an inherent regularization effect that mitigates overfitting on downstream datasets.} This design improves data efficiency and optimization stability under resource-constrained edge settings. Nevertheless, overly small ranks may overly restrict model capacity, leading to slower convergence and degraded task performance. These observations indicate that LoRA rank selection is not merely a hyperparameter choice, but a key system-level control variable that fundamentally impacts learning efficiency and resource utilization in federated fine-tuning, motivating the need for adaptive and energy-aware rank optimization mechanisms.

\begin{figure}[t]
    \centering
    \begin{subfigure}[b]{0.49\columnwidth}
        \centering
        \includegraphics[width=\linewidth]{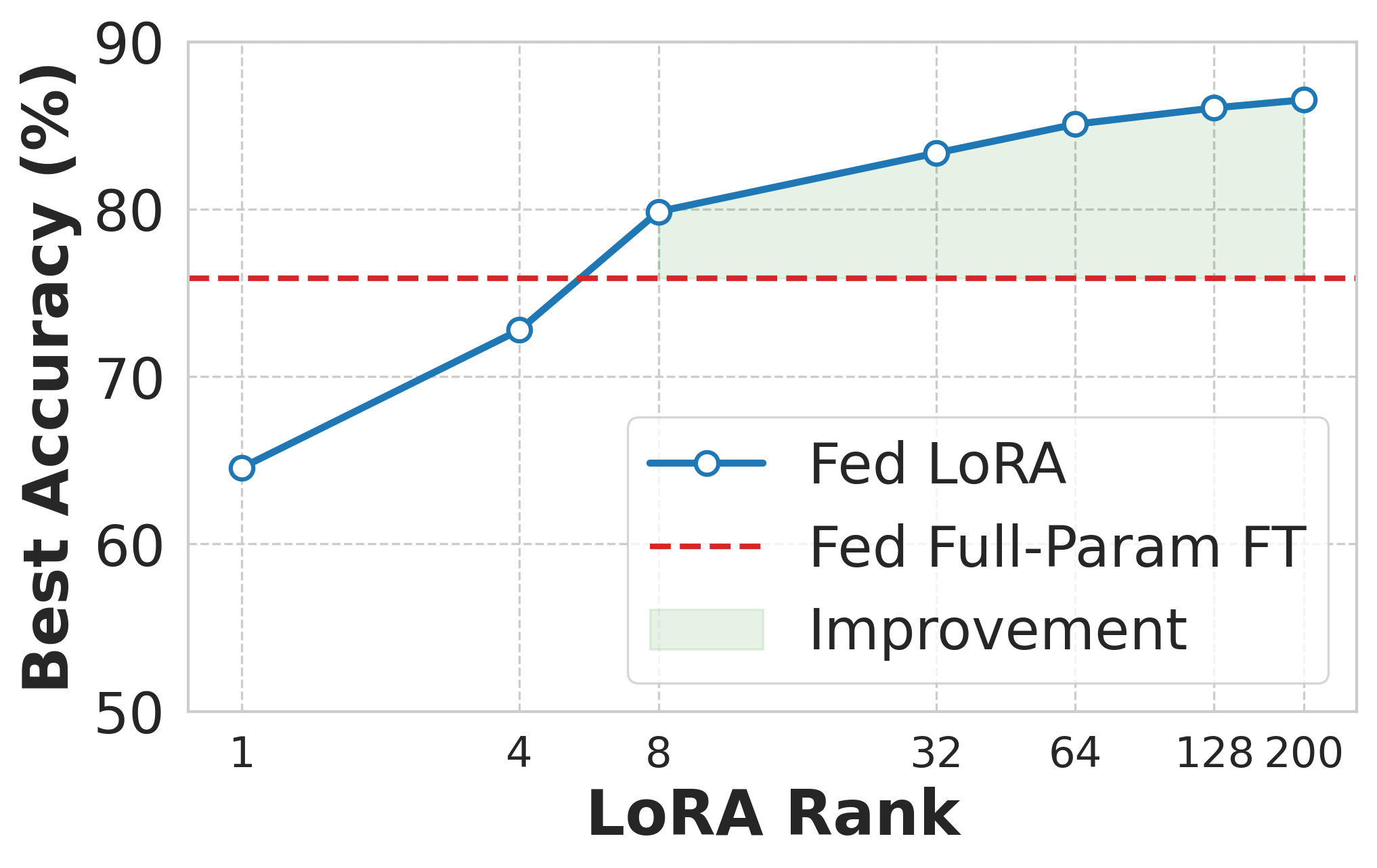}
        \caption{Impact of rank on accuracy.}
        \label{fig:rank_acc_small}
    \end{subfigure}
    \hfill
    \begin{subfigure}[b]{0.49\columnwidth}
        \centering
        \includegraphics[width=\linewidth]{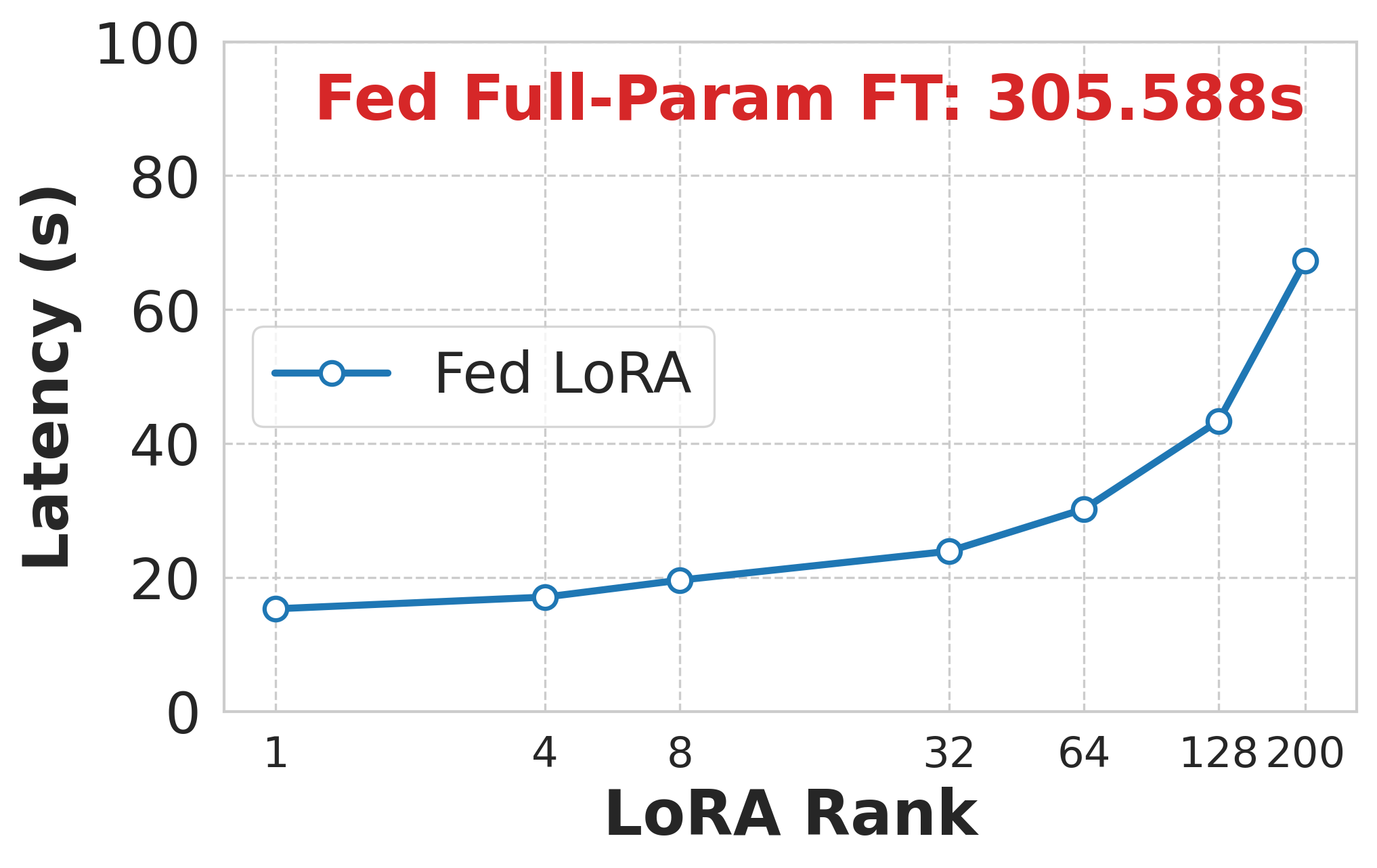}
        \caption{Impact of rank on latency.}
        \label{fig:rank_lat_small}
    \end{subfigure}

    \vspace{0.1cm}

    \begin{subfigure}[b]{0.49\columnwidth}
        \centering
        \includegraphics[width=\linewidth]{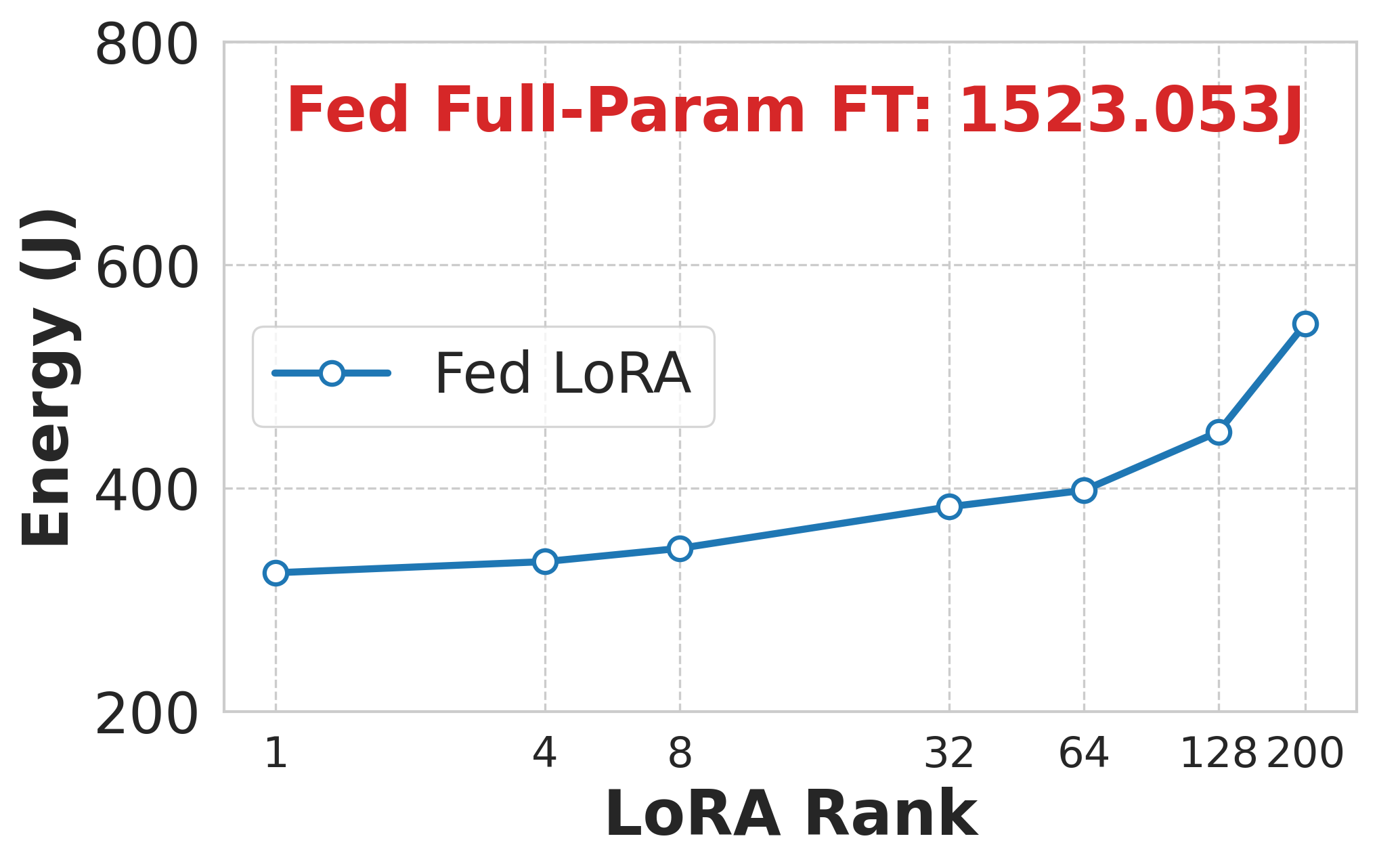}
        \caption{Impact of rank on energy.}
        \label{fig:rank_energy_small}
    \end{subfigure}
    \hfill
    \begin{subfigure}[b]{0.49\columnwidth}
        \centering
        \includegraphics[width=\linewidth]{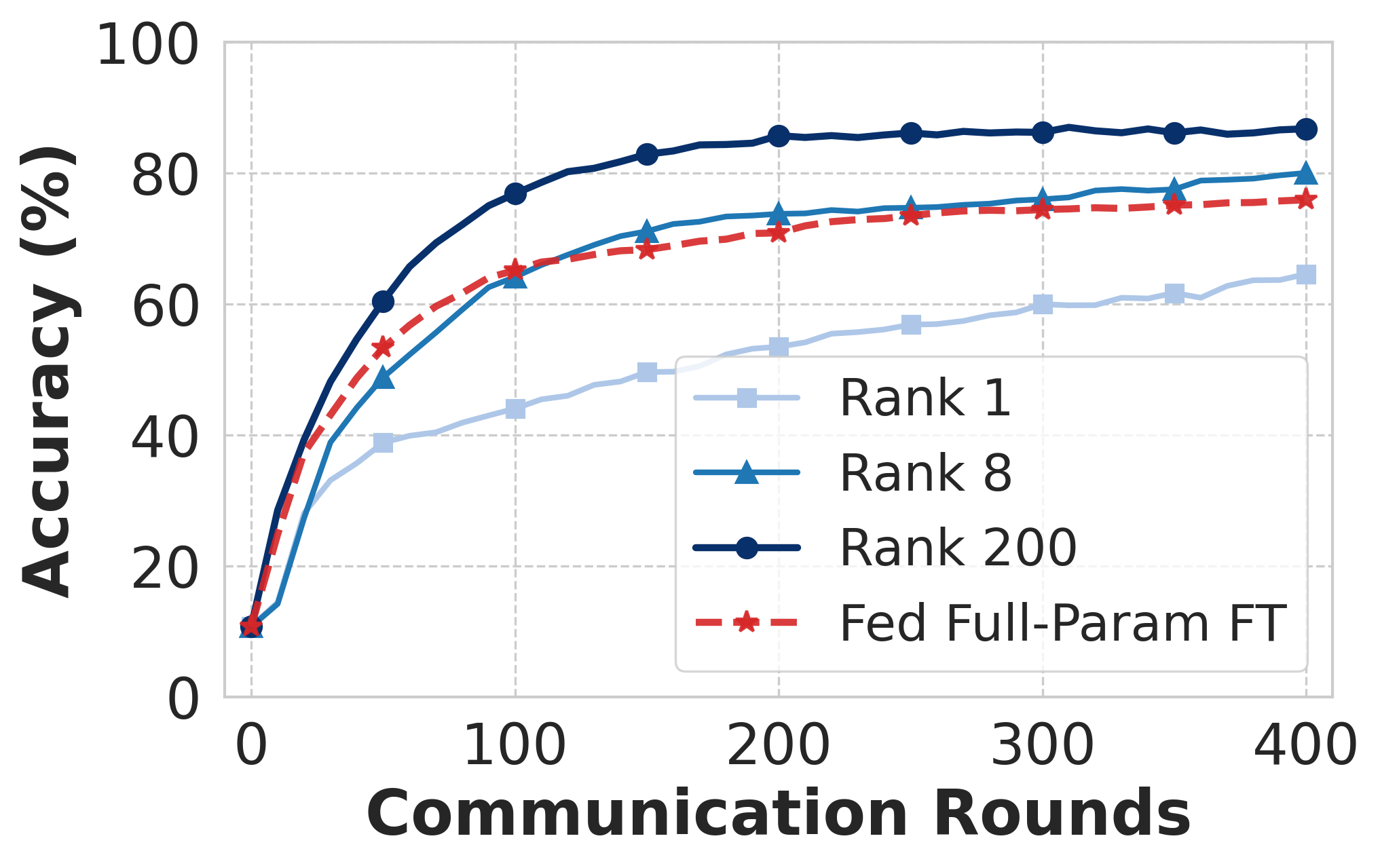}
        \caption{Impact on convergence.}
        \label{fig:rank_conv_small}
    \end{subfigure}
    
    \caption{Impact of LoRA rank on federated fine-tuning (FT) performance. (a) Accuracy, (b) latency, and (c) energy consumption as functions of the LoRA rank, where all clients adopt a uniform rank in the federated setting (denoted as Fed LoRA). Results are compared with federated full-parameter FT (Fed Full-Param FT). (d) Convergence behavior under different ranks, measured by accuracy versus communication rounds. Experiments are conducted using a ViT-based model on a traffic sign classification task.}
    \label{fig:rank_com}
\end{figure}


\noindent {\bf Motivating Scenario.} 
We consider a realistic multi-task federated fine-tuning scenario in IoV systems (Fig.~\ref{fig:intro}), where vehicles concurrently support multiple edge services across RSU coverage areas, a common setting in vehicular edge computing architectures~\cite{meneguette2022vec,kang2024iov_edge}. Multiple RSUs are deployed across an urban area and interconnected through a backbone network, forming a hierarchical learning architecture with cloud-assisted coordination. Each task corresponds to a specific edge service (e.g., urban object detection, road-scene segmentation, or context-aware traffic sign classification~\cite{xu2021blockchain}) and is associated with distinct accuracy, latency, and energy requirements.

\noindent {\bf Global Energy Budget.}
Crucially, on-vehicle learning is governed by fleet-level energy and cost constraints that impose a global energy budget across tasks. Such constraints are common in electric and autonomous vehicle deployments, where operators manage aggregate energy consumption to control charging costs, maintain vehicle availability, and meet service-level objectives across the fleet~\cite{sadeghian2022ev_charging}.
To enforce these limits, a central cloud coordinates task-level energy budgets based on RSU-reported task complexities and system conditions. Critically, this global budget induces a bidirectional coupling between inter-task allocation and intra-task rank adaptation: prioritizing one task inherently constrains the rank-tuning degrees-of-freedom of concurrent tasks, while local rank choices dictate aggregate energy utility. Consequently, task budgeting and model adaptation cannot be resolved in isolation without either violating global energy constraints or significantly sacrificing system-wide efficiency.

Despite its promise, enabling scalable and efficient multi-task FT in IoV faces several {\bf technical challenges}:
\begin{itemize}[leftmargin=*]
\item \textbf{Task and system heterogeneity.}
RSUs must coordinate diverse tasks with heterogeneous accuracy, latency, and energy requirements, across vehicles having disparate computational power and intermittent bandwidth~\cite{Li_Sahu_Talwalkar_Smith_2020,li2020federated}. Balancing multi-task performance under such time-varying conditions necessitates adaptive scheduling and lightweight adaptation strategies to prevent resource bottlenecks and ensure system-wide fairness.
\item \textbf{Distributed adaptation under global energy constraints.} Enforcing global energy budgets is challenging when vehicles make local decisions based on partial observations of data quality and energy status. Since frequent vehicle--RSU interaction is restricted by bandwidth, local adaptation must collectively respect global constraints with minimal signaling. This creates an inherent tension between local performance exploration and coordinated resource sustainability that conventional centralized methods fail to resolve.

\item \textbf{Vehicle mobility and dynamic participation.}
High vehicular mobility results in transient connectivity and frequent client dropouts, often leading to incomplete updates and wasted on-board resources~\cite{wang2022asynchronous}. Instead of assuming persistence, the system must explicitly tolerate premature departures and preserve partial progress. This requires fault-tolerant federated mechanisms capable of handling mobility-induced interruptions with minimal coordination overhead.
\end{itemize}

\smallskip
\noindent {\bf Our Contributions.} To address these challenges, we propose a hierarchical framework for multi-task federated fine-tuning in IoV systems, featuring vehicle-side distributed rank adaptation with lightweight RSU-level coordination. Our main contributions are summarized as follows:

\begin{itemize}[leftmargin=1.6em]
    \item We design a LoRA-based federated fine-tuning framework for real-time multi-task adaptation under mobility, heterogeneity, and energy constraints. Unlike prior edge fine-tuning methods that mainly consider task-wise adaptation or isolated resource control, our framework couples intra-task rank selection with inter-task coordination, enabling vehicles to choose personalized LoRA ranks from local conditions while RSUs aggregate only lightweight feedback without sharing raw data or full model parameters.
    
    \item We formulate LoRA rank adaptation as an energy-constrained online learning problem and develop UCB-DUAL, a primal-dual bandit algorithm for rank selection under dynamic energy budgets. The key design is an energy-aware confidence score that jointly accounts for expected gain, rank-dependent energy cost, and exploration uncertainty; the dual variable is updated at the RSU side using only aggregated scalar energy feedback, which reduces coordination overhead and supports provable sublinear regret and constraint satisfaction.

    \item We develop a large-scale IoV simulator based on real-world urban topologies and the T-Drive trajectory dataset~\cite{zheng2011t-drive}, enabling dynamic modeling of vehicular mobility, intermittent connectivity, and RSU coverage transitions. Our method achieves the best accuracy-efficiency tradeoff among compared baselines, while significantly reducing memory usage, demonstrating superior efficiency and robustness under dynamic conditions.
\end{itemize}

\section{Related Work}
\label{sec:related}

\noindent{\bf Parameter-Efficient Fine-Tuning (PEFT).} PEFT optimizes large-model adaptation under resource constraints by updating a minimal parameter subset~\cite{Houlsby2019Parameter, Zaken2021BitFit, Li2021Prefix}. Notably, LoRA~\cite{hu2021lora} leverages low-rank decomposition to drastically minimize memory and computational overhead.
While recent efforts apply LoRA to edge and personalized learning~\cite{Huang2024Combining}, they typically assume static computation contexts or single-task settings. In contrast, our method extends LoRA with dynamic task-aware rank scheduling, tailored for fluctuating resources and diverse task profiles of IoV systems.

\smallskip
\noindent{\bf Model Fine-Tuning in IoV.}
Recent studies have explored model fine-tuning in vehicular networks. GIOV~\cite{Xie2024GIOV} proposes an RSU-assisted FL framework for resource-constrained adaptation, while GAI-IOV~\cite{Xie2024GAI} leverages generative models for personalized services.
Energy- and latency-aware fine-tuning strategies~\cite{Otoum2022Transfer} improve adaptation efficiency, while Zheng \emph{etal.}~\cite{zheng2025online} addresses joint order-serving and spatio-temporal heterogeneous fine-tuning in vehicle crowdsensing via MARL and GNN-enhanced state modeling.
However, these efforts mostly focus on single-task pipelines or assume static resource partitions. Our approach uniquely targets dynamic multi-task fine-tuning with real-time task scheduling and client mobility support, delivering greater flexibility in IoV deployments with spatiotemporal variation.

\smallskip
\noindent{\bf Federated Fine-Tuning.}
Federated fine-tuning enables collaborative model adaptation while preserving data locality across distributed clients. FedAvg~\cite{mcmahan2017communication} establishes the basic aggregation framework, and subsequent works primarily address system and client heterogeneity.
While early research addressed system heterogeneity via regularization and personalization~\cite{li2020federated, arivazhagan2019federated}.
More recently, PEFT has been integrated into FL to accommodate resource-constrained clients. For instance, HetLoRA~\cite{cho2024heterogeneous} employs heterogeneous ranks via zero-padding, and FedRA~\cite{su2024fedra} utilizes randomized layer allocation. However, these methods often assume static environments and lack explicit modeling of task-level energy budgets or client mobility. Similarly, Tomtit~\cite{Tomtit2024} optimizes adapter synchronization via reinforcement learning but is limited to single-task scenarios.
In contrast, our work targets multi-task federated fine-tuning in dynamic IoV environments, explicitly accounting for client mobility, heterogeneous resources, and task-level energy budgets. We formulate LoRA rank adaptation as a constrained online learning problem and develop a lightweight primal--dual UCB-based mechanism for energy-aware rank selection with minimal coordination overhead.

\smallskip
\noindent{\bf Constrained Bandits and Online Resource Allocation.} While constrained online learning paradigms commonly integrate UCB with primal-dual updates~\cite{badanidiyuru2018bandits, qiu2020upper}, they typically assume stationary environments with independent atomic arms. LoRA rank selection, however, presents correlated reward-cost structures, rendering such atomic models sub-optimal. Furthermore, unlike multi-agent methods relying on independent local constraints~\cite{zhang2024rosevin}, we must manage a coupled global energy budget across mobile, heterogeneous clients. Our UCB-DUAL algorithm addresses this by coupling inter-task energy allocation with intra-task rank adaptation, enforcing global constraints via lightweight, scalar feedback specifically tailored to the dynamics of federated fine-tuning.
\section{Problem Formulation}
\label{sec: problem}
\subsection{System Overview}
\label{subsec:sys-overview}

We propose a multi-task federated fine-tuning system tailored for edge-assisted IoV, comprising $K$ RSUs and $V$ intelligent vehicles. Each task $t \in \mathcal{T}$ corresponds to an independent fine-tuning objective executed via distributed collaboration, where $\mathcal{V}_t$ denotes the set of vehicles participating in task $t$. To handle dynamic task demands and heterogeneous system constraints, we adopt a hierarchical optimization framework that jointly addresses resource scheduling and rank adaption.

In the \textbf{single-task setting}, an adaptive-rank LoRA-based fine-tuning framework allows each vehicle to select a suitable rank $\eta_v^t$ based on its device capacity and task demands, enhancing local efficiency. In the \textbf{multi-task setting}, cross-RSU coordination enables global resource sharing among concurrent tasks. During execution, each vehicle $v$ optimizes fine-tuning accuracy $q_v^t$ and latency $\tau_v^t$ under its energy constraint. This design captures the trade-off between performance and cost, supporting adaptive optimization in distributed, resource-constrained environments.

\subsection{Adaptive-Rank LoRA-Based Federated Fine-Tuning}

\begin{figure}
    \centering
    \includegraphics[width=1\linewidth]{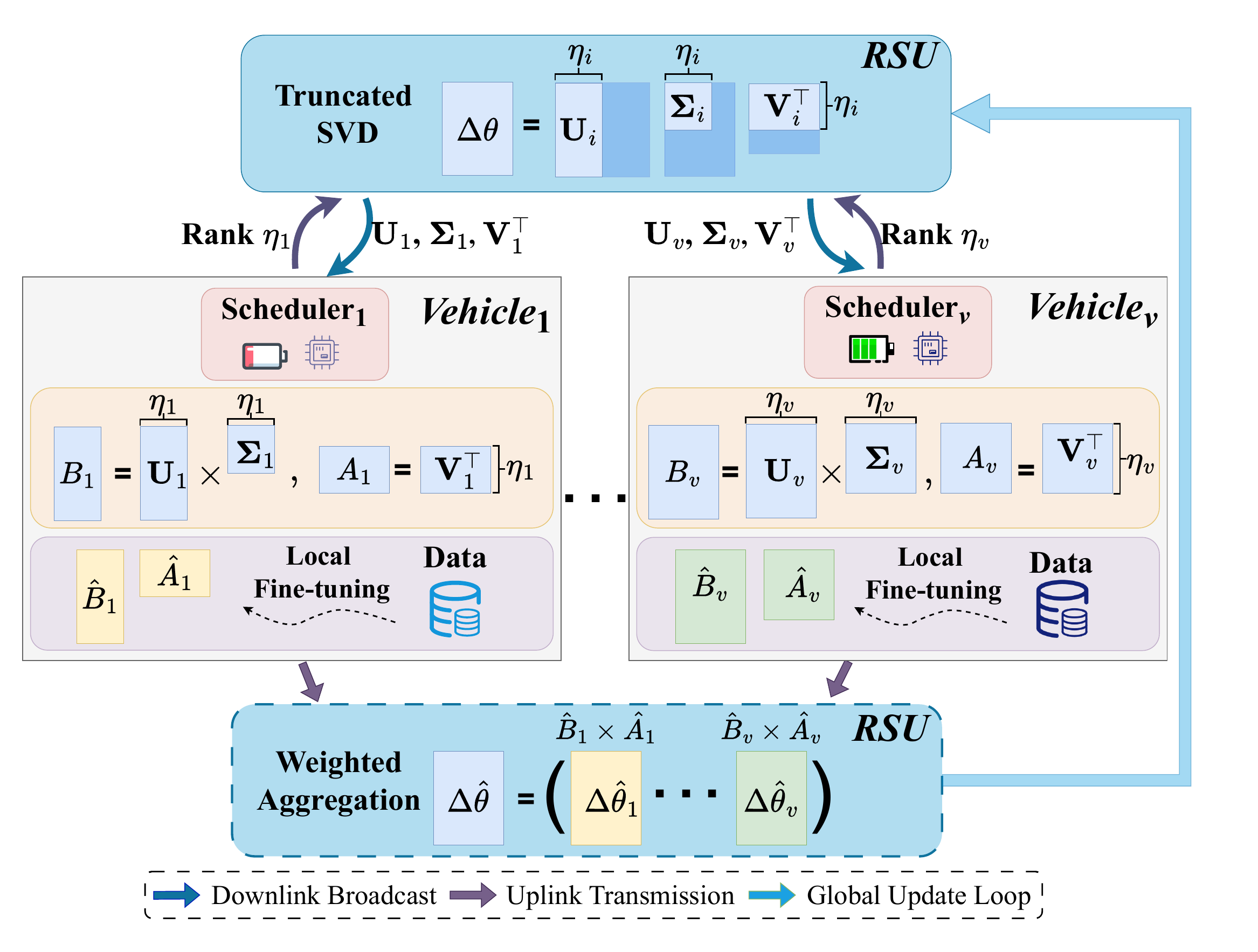}
    \caption{Overview of the adaptive-rank LoRA-based federated fine-tuning framework. The RSU performs truncated SVD on the global adapter $\Delta \theta$. Vehicles receive personalized components based on their selected ranks $\eta$, constructing local adapters $(B, A)$. After local fine-tuning, vehicles upload the updated matrices $(\hat{B}, \hat{A})$ to the RSU for weighted aggregation and reconstruction of the new global adapter $\Delta \hat{\theta}$.}
    \label{fig:AFlora}
\end{figure}
To enable efficient federated fine-tuning of foundation models in resource-constrained and heterogeneous IoV environments, we adopt LoRA as a parameter-efficient adaptation mechanism and extend it with adaptive rank allocation.
Specifically, the fine-tuned model parameters $\theta$ are represented as $\theta = \theta_0 + \Delta \theta$, where $\theta_0 \in \mathbb{R}^{d_1 \times d_2}$ denotes the frozen pre-trained backbone and $\Delta \theta$ is the trainable LoRA adapter.
Instead of updating $\theta_0$ directly, LoRA represents task-specific updates via a low-rank decomposition $\Delta \theta = BA$, where $B \in \mathbb{R}^{d_1 \times \eta}$ and $A \in \mathbb{R}^{\eta \times d_2}$ with $\eta \ll \min(d_1,d_2)$.

A key limitation of standard LoRA-based federated fine-tuning is the use of a fixed rank across all clients, which ignores heterogeneity in task demands and device capabilities. To address this, we propose an adaptive-rank LoRA-based federated fine-tuning framework, illustrated in Fig.~\ref{fig:AFlora}. The RSU first computes the global LoRA adapter parameters $\Delta \theta$, rather than the full pre-trained model weights, and applies Singular Value Decomposition (SVD), obtaining $\Delta \theta = \mathbf{U} \mathbf{\Sigma} \mathbf{V}^\top$. Here, $\mathbf{U}$ and $\mathbf{V}$ are orthogonal matrices representing the left and right singular vectors, respectively, while $\mathbf{\Sigma}$ is a diagonal matrix containing the singular values in descending order. These components are computed from the aggregated $\Delta \theta$ to capture the principal directions and relative intensities of the model update, providing the basis for subsequent rank truncation.
Each vehicle $v$ then selects a personalized rank $\eta_v$ based on its resource budget and task characteristics, using strategies designed in this work, reconstructing its local low-rank adapter using the truncated SVD components:
\begin{align*}
\Delta \theta_v = B_v A_v, \quad \text{where} \quad 
\begin{cases} 
B_v = \mathbf{U}[:, :\eta_v] \mathbf{\Sigma}[: \eta_v, :\eta_v], \\
A_v = \mathbf{V}[:, :\eta_v]^\top.
\end{cases}
\end{align*}

This formulation allows each vehicle $v$ to adaptively scale model complexity by tuning the update granularity to its own capability. After local fine-tuning with the selected adapter, vehicles upload their updated matrices $(\hat{B}_v, \hat{A}_v)$ to the RSU. The RSU performs weighted aggregation based on local data volumes to update the global model as:
\begin{align*}
    \Delta \hat{\theta} = \sum_{v \in \mathcal{V}_t} \frac{|D_v|}{|D|} \hat{B}_v \hat{A}_v, \quad \text{where} \quad |D| = \sum_{v \in \mathcal{V}_t} |D_v|.
\end{align*}

Here, $D_v$ denotes the local dataset of vehicle $v$, and $|D_v|$ represents its size.
This weighted scheme ensures that each vehicle's contribution reflects the significance of its local data, preserving fairness and robustness in global adaptation.

\noindent\textbf{Feasibility of SVD Truncation:} The diagonal entries of $\Sigma$ are ordered by magnitude, with larger singular values capturing dominant structures in the model update. Truncating at rank $\eta_v$ retains high-energy components while discarding redundant or noise-prone directions, yielding an effective trade-off between performance and resource efficiency, which is particularly critical for latency-sensitive and energy-bounded IoV settings.

\noindent\textbf{Computational Overhead Analysis.}
Let $\Delta \hat{\theta} \in \mathbb{R}^{d_1 \times d_2}$ be the updated LoRA matrix.
While full SVD requires $\mathcal{O}(\min(d_1 d_2^2, d_1^2 d_2))$, we employ truncated SVD to extract the leading $\eta_{\max}$ singular components, where $\eta_{\max}$ denotes the maximum rank in the candidate set. This reduces the computational complexity to $\mathcal{O}(d_1 d_2 \eta_{\max})$, an operation executed once per global round at the RSU and amortized across all vehicles, incurring negligible latency relative to local training.
Crucially, this centralized overhead is independent of the vehicle count. By distributing only the essential truncated factors, our design facilitates efficient, heterogeneous adaptation while minimizing redundant computation and transmission overhead.

This framework provides a unified and practical foundation for implementing heterogeneous and dynamic rank control, which is essential for the hierarchical adaptive rank selection algorithm introduced in section~\ref{sec:algorithm}.

\subsection{Modeling the Single-Task Optimization Problem}
\label{sec:single}

In the single-task case, we aim to adaptively determine the optimal LoRA rank $\eta_v^t$ for each vehicle $v$ participating in a specific task $t$. Our goal is to minimize a fleet-level cost-accuracy objective for federated LoRA, under per-vehicle energy and wall-clock latency constraints.
This optimization problem characterizes heterogeneity across vehicles in terms of computation capacity, communication bandwidth, and local data availability. Such heterogeneity necessitates a task-aware and resource-constrained optimization strategy.
Formally, the optimization objective is defined as:
\begin{align*}
\min_{\eta_v^t, \Delta \theta^t} \left[ \alpha \left( \max_{v \in \mathcal{V}_t} \tau_v^t \right) - \gamma \left( \frac{1}{V_t} \sum_{v \in \mathcal{V}_t} q_v^t \right) \right], \quad  \forall t \in \mathcal{T},
\end{align*}
$$
\text{s.t.} \quad 
\begin{cases} 
E_v^t \leq E_{\max}, &  \text{(C1: Energy constraint)}, \\
\tau_v^t \leq \tau_{\text{SLA}}^t, & \text{(C2: Latency constraint)}, \\
\Delta \hat{\theta}^t = \sum_{v \in \mathcal{V}_t} \frac{|D_v|}{|D|}\hat{B}_v^t \hat{A}_v^t, & \text{(C3: Aggregation)}.
\end{cases}
$$

The objective captures the key tradeoff: the maximum latency $\max \tau_v^t$ determines synchronization delay, while the average accuracy reflects overall model quality. Scalars $\alpha$ and $\gamma$ control their relative importance.
Constraints (C1)--(C3) strictly enforce the per-vehicle energy limits, stringent task latency requirements, and standard federated aggregation protocols.

To thoroughly analyze the latency and energy consumption characteristics of the system, we decompose each communication round of federated fine-tuning into four distinct stages:

\textbf{(1) Model Distribution.}
The RSU transmits personalized truncated SVD components $\{\mathbf{U}_v^t, \mathbf{\Sigma}_v^t, (\mathbf{V}_v^t)^\top\}$ to vehicle $v$ according to its selected rank $\eta_v^t$. This reduces the downlink payload to $\Omega_v^{(d)} \approx \eta_v^t(d_1 + d_2)$, achieving a compression ratio of $\frac{d_1 d_2}{\eta_v^t(d_1+d_2)}$ compared to full-parameter broadcasting.
The transmission latency and energy consumption are:
\begin{align*}
\tau_{v,k}^{(d)} = \frac{\Omega_v^{(d)}}{R_{v,k}^{(d)}}, \quad E_{v,k}^{(d)} = p_{v,k} \cdot \tau_{v,k}^{(d)},
\end{align*}
where $p_{v,k}$ is the RSU transmit power, and the downlink rate $R_{v,k}^{(d)} = W \log_2(1 + \text{SINR}_{v,k})$ is determined by the Shannon capacity model~\cite{Tse_Viswanath_2005}. Here, $W$ represents the system bandwidth, and $\text{SINR}_{v,k}$ is the signal-to-interference-plus-noise ratio, which accounts for distance-dependent path loss and small-scale Rayleigh fading.

\textbf{(2) Local Fine-tuning.}
Vehicle $v$ fine-tunes the global model using its local dataset (size $D_v^t$). The computation latency and energy are formulated as:
\begin{align*}
    \tau_v^{\text{comp}} = \frac{C_v \cdot D_v^t \cdot g(\eta_v^t)}{f_v}, \quad E_v^{\text{comp}} = \kappa_v \cdot (f_v)^3 \cdot \tau_v^{\text{comp}},
\end{align*}
where $C_v$ denotes the computation per sample, $g(\eta_v^t)$ is a rank-dependent complexity function, $f_v$ is the device frequency, and $\kappa_v$ is the system energy coefficient.

\textbf{(3) Parameter Upload.}
Vehicles upload the updated factor matrices $\{\hat{B}_v^t, \hat{A}_v^t\}$ to the RSU. With an uplink payload $\Omega_v^{(u)} = \eta_v^t(d_1 + d_2)$, the latency and energy costs are:
$$
    \tau_{v,k}^{(u)} = \frac{\Omega_v^{(u)}}{R_{v,k}^{(u)}}, \quad
    E_{v,k}^{(u)} = p_v \cdot \tau_{v,k}^{(u)},
$$
where $p_v$ is the vehicle's transmit power, and $R_{v,k}^{(u)}$ is the uplink rate derived from the Shannon capacity model. 

\textbf{(4) Aggregation.}
The RSU aggregates updates from $V$ vehicles. The corresponding latency and energy costs are:
$$
    \tau_k^{\text{agg}} = \frac{C_{\text{agg}} \cdot V}{f_k}, \quad E_k^{\text{agg}} = \kappa_k \cdot (f_k)^3 \cdot \tau_k^{\text{agg}},
$$
where $C_{\text{agg}}$ is the computation density of aggregation, $f_k$ is the RSU's computing power, $V$ is the number of vehicles, and $\kappa_k$ is the RSU's energy coefficient.

This stage-wise decomposition quantifies the fine-grained costs of federated fine-tuning, providing a rigorous analytical foundation to balance latency, energy, and accuracy trade-offs.

\subsection{Modeling the Multi-Task Optimization Problem}

We then consider a multi-task setting motivated by practical fleet operations, where vehicles have independent batteries but on-vehicle learning is governed by fleet-level energy and cost constraints that impose a global energy budget across tasks. Such fleet-level constraints are common in electric and autonomous vehicle deployments~\cite{sadeghian2022ev_charging}, where operators manage aggregate energy consumption to control charging costs, maintain vehicle availability, and meet operational service targets across the fleet. Accordingly, energy must be jointly allocated across tasks rather than optimized in isolation.

Building upon the four-stage models in Section~\ref{sec:single}, we define the wall-clock latency $\tau_t$, total energy consumption $E_t$, and average fine-tuning accuracy $q_t$ for each task $t \in \mathcal{T}$:
\begin{align}
\tau_t
&= \max_{v \in \mathcal{V}_t} \tau_{v,k}^{(d)}
 + \max_{v \in \mathcal{V}_t} \tau_v^{\text{comp}}
 + \max_{v \in \mathcal{V}_t} \tau_{v,k}^{(u)}
 + \tau_k^{\text{agg}},
\\
E_t
&= \sum_{v \in \mathcal{V}_t} E_{v,k}^{(d)}
 + \sum_{v \in \mathcal{V}_t} E_v^{\text{comp}}
 + \sum_{v \in \mathcal{V}_t} E_{v,k}^{(u)}
 + E_k^{\text{agg}},
\label{eq:task_energy}
\\
q_t
&= \frac{1}{V_t} \sum_{v \in \mathcal{V}_t} q_v^t,
\label{eq:task_acc}
\end{align}
where $q_v^t$ is the fine-tuning accuracy of vehicle $v$ on task $t$, and $\mathcal{V}_t$ (with fleet size $V_t = |\mathcal{V}_t|$) denotes the set of participating vehicles. 


To prevent any single task from monopolizing resources, we aim to maximize the system-wide utility—a weighted sum of accuracy and latency—subject to a global energy budget $E_{\text{total}}$. The multi-task optimization problem is formulated as follows:
\begin{align}
\max_{\eta_v^t, \Delta \theta^t} \sum_{t \in \mathcal{T}} \left[\gamma  q_t - \alpha \tau_t\right],
\label{eq:multi_task_opt}
\end{align}
\begin{align*}
\text{s.t.} \quad
\begin{cases}
\sum_{t \in \mathcal{T}}E_t \leq E_{\text{total}}, \\
\text{C1: Hardware constraint}, \\
\text{C2: Latency constraint}, \\
\text{C3: Federated aggregation}.
\end{cases}
\end{align*}


Here, $\gamma$ and $\alpha$ are user-defined weights balancing model accuracy and delay, respectively. This formulation ensures efficient inter-task coordination by dynamically reallocating resources based on task priorities while strictly complying with aggregate energy limitations and per-task operational requirements.

\section{UCB-DUAL-Based Adaptive Rank Allocation}
\label{sec:algorithm}

\subsection{Overview of Hierarchical Optimization}
In Section~\ref{sec: problem}, we introduced the architecture for multi-task adaptive LoRA-based federated fine-tuning, highlighting the benefits of rank flexibility in balancing performance and efficiency. However, it is difficult to select appropriate LoRA ranks per vehicle, as rank decisions directly impact accuracy, latency, and energy consumption. To address this, we propose a hierarchical adaptive rank allocation framework that jointly optimizes task-level and vehicle-level decisions to maximize the multi-objective function in Equation~\ref{eq:multi_task_opt}. 
Specifically:
\begin{itemize}
    \item \textbf{Inter-Task (RSU-Side):} A global energy budget is dynamically distributed across tasks based on feedback signals such as task difficulty and energy utilization.
    \item \textbf{Intra-Task (Vehicle-Side):} Given the task-specific budget, a distributed UCB-DUAL algorithm is employed to assign ranks to individual vehicles, balancing exploration and exploitation under per-task constraints.
\end{itemize}

\subsection{Inter-Task: Budget Allocation across Tasks}
In the multi-task federated fine-tuning framework, we aim to allocate a total energy budget $E_{\text{total}}$ across $T$ heterogeneous tasks. Task-specific energy demands vary due to differences in complexity and resource efficiency.
To this end, we introduce a two-level feedback-driven energy allocation mechanism that periodically redistributes energy based on task difficulty and energy utilization.

To quantify the difficulty of task~$t$ at round~$m$, we define a smoothed task difficulty coefficient $h_t^{m} \in (0, 1]$ using an exponential moving average (EMA) approach:
\begin{align}
h_t^{m} = \xi h_t^{m-1} + (1 - \xi) \left( \frac{\overline{E}_t^{m}}{q_t^{m}} \right),
\label{eq:difficult}
\end{align}
where $q_t^m$ is the average fine-tuning accuracy of task $t$ across all participating vehicles in round $m$ (see Eq.~\ref{eq:task_acc}), $\overline{E}_t^m$ denotes the energy budget allocated to task $t$ in the same round, and $\xi \in [0, 1]$ is a temporal smoothing factor.
This ratio captures the energy cost per unit accuracy, reflecting task difficulty over time.

We define energy utilization efficiency $\mu_t^{(m)} \in [0, 1]$ as:
\begin{align}
\mu_t^{m} = \frac{E_t^{m}}{\overline{E}_t^{m}},
\label{eq:utility}
\end{align}
where $E_t^m$ is the actual energy consumed during the fine-tuning of task $t$ in round $m$ (as derived in Eq.~\ref{eq:task_energy}), and $\overline{E}_t^m$ is the previously defined energy budget.
This efficiency metric identifies tasks that are over-provisioned or under-utilized, providing a feedback signal for the subsequent budget reallocation process.

We combine difficulty and utilization into a priority weight:
\begin{align}
w_t^{m} = \left(h_t^{m}\right)^\zeta \cdot \mu_t^{m},
\label{eq:weight}
\end{align}
where $\zeta > 1$ amplifies the priority of difficult tasks.

These feedback signals are integrated into Algorithm~\ref{alg:dynamic-energy}, which outlines the periodic energy reallocation process. The algorithm begins with an equal division of the total energy budget, incorporating minor rounding adjustments. Every $Q$ rounds, it updates task difficulty, utilization, and composite weight (lines~\ref{line:mod-check} -- \ref{line:com_weight}), then reallocates the remaining budget proportionally to $w_t^m$, with a cap of $0.7E_{\text{total}}$ per task(lines~\ref{line:task} -- \ref{line:energy}). In other rounds, energy assignments remain unchanged (lines~\ref{line:remain}).
The resulting budgets constrain intra-task rank selection, enabling coordinated energy-aware fine-tuning under a global constraint.

\begin{algorithm}[t]
\caption{Dynamic Task-Level Energy Allocation}
\label{alg:dynamic-energy}
\SetAlgoVlined
\SetKwInOut{Input}{Input}
\SetKwInOut{Output}{Output}
\SetKwInOut{Initialize}{Initialize}
\Input{Total energy budget $E_{\text{total}}$, number of tasks $T$, warm-up period $Q$}
\Output{Energy allocation sequence $\{\overline{E}_t^{m}\}_{m=1}^M$}
\Initialize{
$\overline{E}_t^{0} \leftarrow \left\lfloor \frac{E_{\text{total}}}{T} \right\rfloor + \mathbb{I}(t \leq E_{\text{total}} \bmod T),\quad \forall t \in \mathcal{T}$
}

\For{$m = 1$ \KwTo $M$}{
    \If{$m \bmod Q = 0$\label{line:mod-check}}{
        \ForEach{task $t \in \mathcal{T}$}{
            Update difficulty: $h_t^{m}$ via Eq.~\eqref{eq:difficult}\;
            Compute utilization: $u_t^{m}$ via Eq.~\eqref{eq:utility}\;
            Calculate weight: $w_t^{m}$ via Eq.~\eqref{eq:weight}\label{line:com_weight}\;
        }
        Compute remaining energy: $\overline{E}_{\text{rem}} = E_{\text{total}} - \sum_{t=1}^T \overline{E}_t^{m}$\;
        \ForEach{task $t \in \mathcal{T}$\label{line:task}}{
            $\Delta \overline{E}_t \leftarrow \text{round}\left( \frac{w_t^{m} \cdot \overline{E}_{\text{rem}}}{\sum_{j=1}^T w_j^{m}} \right)$\;
            Update energy allocation:
            $\overline{E}_t^{m+1} = \min\left( \overline{E}_t^{m} + \Delta \overline{E}_t,\; 0.7 \cdot E_{\text{total}} \right)$\label{line:energy}\;
        }
    }
    \Else{
        $\overline{E}_t^{m+1} \leftarrow \overline{E}_t^{m},~\forall t$\label{line:remain}\;
    }
}
\end{algorithm}

\subsection{Intra-Task: Rank Allocation via UCB-DUAL Algorithm}

At the intra-task level, we optimize rank allocation subject to energy constraints by modeling the process as a distributed constrained multi-armed bandit (MAB) problem. In each round $m$, vehicle $v \in \mathcal{V}_t$ selects a rank $\eta_v^m$ from a discrete set $\phi_\eta$ to maximize the cumulative reward:
\begin{align*}
R_v^m(\eta) = -\alpha \cdot \tau_v^m(\eta) + \gamma \cdot q_v^m(\eta),
\end{align*}
where $q_v^m$ and $\tau_v^m$ denote local fine-tuning accuracy and latency, respectively. By solving this online bandit problem, vehicles autonomously optimize their trade-off between accuracy and delay, ensuring local decisions align with the global objectives in \eqref{eq:multi_task_opt}.

\noindent{\bf Handling energy constraints via primal-dual theories.} Each configuration incurs an energy cost $E_v^m(\eta)$, subject to:
\begin{align*}
\sum_{v \in \mathcal{V}_t} E_v^m(\eta_v^m) \leq \overline{E}_t^m, \quad \forall m.
\end{align*}

To handle this constraint, we apply Lagrangian relaxation with dual variable $\lambda^m \geq 0$, yielding the decoupled objective:
\begin{align*}
\mathcal{L}^m = \sum_{v \in \mathcal{V}_t} \left[ R_v^m(\eta_v^m) - \lambda^m E_v^m(\eta_v^m) \right] + \lambda^m \overline{E}_t^m.
\end{align*}

Each vehicle then independently selects its rank via:
\begin{align*}
\eta_v^m = \arg\max_{\eta \in \phi_\eta} \left[ R_v^m(\eta) - \lambda^m E_v^m(\eta) + \text{UCB}_v^m(\eta) \right],
\end{align*}
where the exploration bonus is given by:
\begin{align*}
\text{UCB}_v^m(\eta) = \epsilon \cdot \sqrt{\frac{\log m}{1 + N_v^m(\eta)}},
\end{align*}
where $N_v^m(\eta)$ denotes the selection count of rank $\eta$ by vehicle $v$.
Dual variables update via projected subgradient ascent:
\begin{align*}
\lambda^{m+1} = \left[ \lambda^m + \omega \left( \sum_{v \in \mathcal{V}_t} E_v^m(\eta_v^m) - \overline{E}_t^m \right) \right]_+.
\end{align*}

These objectives are implemented in Algorithm~\ref{alg:dual-ucb}, which outlines the distributed rank allocation process via our proposed UCB-DUAL method. The algorithm initializes the dual variable and rank usage counters. In each round, vehicles independently estimate reward, energy cost, and UCB bonus for each candidate rank (lines~\ref{line:re_co} -- \ref{line:UCB_com}), and select the configuration that maximizes a utility function combining performance, penalized energy consumption, and exploration (line~\ref{line:select-rank}). Rank counts are updated accordingly. The RSU then updates the dual variable using subgradient ascent based on total energy usage (line~\ref{line:dual-update}), enforcing the global energy constraint. UCB-DUAL supports online, energy-aware rank adaptation across vehicles with lightweight RSU-level coordination, providing scalable and robust optimization for dynamic federated environments.

\begin{algorithm}[t]
\caption{UCB-DUAL-Based Rank Allocation}
\label{alg:dual-ucb}
\SetAlgoVlined
\SetKwInOut{Input}{Input}
\SetKwInOut{Output}{Output}
\SetKwInOut{Initialize}{Initialize}
\SetKwFunction{Select}{Select}
\Input{Action space $\phi_\eta$, learning rate $\omega$, exploration factor $\epsilon$, energy budget $E_{\text{total}}$}
\Output{Selected ranks $\{\eta_v^m\}$ for each $v \in \mathcal{V}_t$ and round $m$}
\Initialize{$\lambda^1 \leftarrow 0$, $N_v^0(\eta) \leftarrow 0,~\forall v\in\mathcal{V}_t,~\eta\in\phi_\eta$}

\For{$m=1$ \KwTo $M$}{
    \ForEach{vehicle $v \in \mathcal{V}_t$ \textbf{(in parallel)}\label{line:re_co}}{  
        \ForEach{$\eta \in \phi_\eta$}{ 
            Estimate reward $R_v^m(\eta)$ and cost $E_v^m(\eta)$\;
            Compute UCB bonus: $\text{UCB}_v^m(\eta) = \epsilon \sqrt{\frac{\ln m}{N_v^{m-1}(\eta) + 1}}$\label{line:UCB_com}\;    
        }
        Select configuration:
        $\eta_v^m = \arg\max_{\eta \in \phi_\eta} \left[ R_v^m(\eta) - \lambda^m E_v^m(\eta) + \text{UCB}_v^m(\eta) \right]$\label{line:select-rank}\;
        Update count: $N_v^m(\eta_v^m) = N_v^{m-1}(\eta_v^m) + 1$\;
    }
    Update dual variable:
    $\lambda^{m+1} = \left[ \lambda^m + \omega \left( \sum_{v \in \mathcal{V}_t} E_v^m(\eta_v^m) - \overline{E}_{t}^m \right) \right]_+$\label{line:dual-update}\;
}
\end{algorithm}

\subsection{Theoretical Guarantees}
We analyze the UCB-DUAL algorithm by characterizing its cumulative regret and energy constraint violation. Note that while vehicles are coupled through the dual variable $\lambda^m$, the system exhibits conditional separability given the sequence $\{\lambda^m\}_{m=1}^M$.

\begin{theorem}
\label{thm:ucb_dual}
With a learning rate $\omega = \Theta(1/\sqrt{M})$, the \emph{UCB-DUAL} algorithm achieves cumulative regret 
\begin{align*}
\text{Regret}_{\text{total}}(M) = \mathcal{O}(V |\phi_\eta| \sqrt{M \ln M}),
\end{align*}
and expected energy violation
\begin{align*}
\mathbb{E}[\mathbb{V}(M)] = \mathcal{O}(\sqrt{M}),
\end{align*}
where $\mathbb{V}(M) := \sum_{m=1}^M \left[\sum_v E_v^m(\eta_v^m) - \overline{E}_t^m\right]_+$ denotes total energy violation, and $|\phi_\eta|$ is the size of the candidate rank set.
\end{theorem}

\begin{proof}[Proof Sketch]
We decompose the analysis into three parts.

\noindent\textbf{(1) Conditional Per-vehicle regret.}
Conditioned on a fixed dual sequence $\{\lambda^m\}$, the vehicle-level optimization problem decouples. We define the dual-regularized reward as $\tilde{R}_v^m(\eta) := R_v^m(\eta) - \lambda^m E_v^m(\eta)$. Given $\{\lambda^m\}$, each vehicle independently solves a non-stationary bandit problem. Let $\eta^*_v$ be the best fixed action in hindsight for vehicle $v$. The conditional regret is:
\begin{align*}
\text{Regret}_v(M) := \sum_{m=1}^M \tilde{R}_v^m(\eta^*) - \sum_{m=1}^M \tilde{R}_v^m(\eta_v^m).
\end{align*}

Applying the UCB policy with Hoeffding's inequality, we obtain:
\begin{align*}
\text{Regret}_v(M) \leq 4c\sqrt{|\phi_\eta| M \ln M} + \left(\frac{\pi^2}{3} + 1\right) |\phi_\eta| \Delta_{\max},
\end{align*}
where $\Delta_{\max} := \max_{\eta} [\tilde{R}_v^m(\eta^*) - \tilde{R}_v^m(\eta)]$.
Thus, the local regret grows sublinearly at $\mathcal{O}(\sqrt{M \ln M})$.

\noindent\textbf{(2) Global regret.} Summing over all vehicles yields the total regret:
\begin{align*}
\text{Regret}_{\text{total}}(M) := \sum_{v \in \mathcal{V}} \text{Regret}_v(M) = \mathcal{O}(V |\phi_\eta| \sqrt{M \ln M}).
\end{align*}

Although this bound is established against a static comparator, the sublinear growth rate implies the algorithm's ability to track the shifting optimum in non-stationary IoV environments.

\noindent\textbf{(3) Constraint violation.}
Let $\delta^m := \sum_{v} E_v^m(\eta_v^m) - \overline{E}_t^m$ denote the aggregate energy deviation. The dual update $\lambda^{m+1} = [\lambda^m + \omega \delta^m]_+$ follows projected subgradient ascent. Defining the Lyapunov function $L^m := \frac{1}{2} (\lambda^m)^2$, its drift is bounded by:
\begin{align*}
L^{m+1} - L^m \leq \omega \lambda^m \delta^m + \frac{1}{2} \omega^2 (\delta^m)^2.
\end{align*}

Summing over $m=1, \dots, M$ and using $\lambda^{M+1} \geq \omega \mathbb{V}(M)$ yields:
\begin{align*}
\frac{1}{2} \omega^2 \mathbb{V}(M)^2 \leq L^{M+1} \leq \omega \sum_{m=1}^M \lambda^m \delta + \frac{1}{2} \omega^2 M V^2 E_{\max}^2.
\end{align*}

Using the bound $\lambda^m \leq \frac{V E_{\max}}{\omega}$, we obtain:
\begin{align*}
\mathbb{E}[\mathbb{V}(M)] \leq \frac{\sqrt{2 \mathbb{E}[L^{M+1}]}}{\omega} = \mathcal{O}(\sqrt{M}),
\end{align*}
which completes the proof.
\end{proof}

\subsection{Mobility-Aware Fault-Tolerant Scheduling}

In vehicular federated fine-tuning, another challenge is that vehicles may disconnect from the RSU before completing training, wasting computation and losing updates. We then propose a mobility-aware fault-tolerant scheduling strategy that anticipates disconnections and selects a cost-minimizing fallback action. For vehicle $v$ on task $t$ at round $m$, the system evaluates three fallback strategies upon predicted departure:

\noindent\textbf{Strategy 0 (Early Upload):} If the local accuracy $q_v^{t,m}$ exceeds a threshold $q_v^{*,t}$, the vehicle uploads its LoRA parameters immediately. The residual loss is penalized by
\begin{align*}
\text{Cost}_0 = \gamma \cdot \max(0, q_v^{*,t} - q_v^{t,m}).
\end{align*}

\noindent\textbf{Strategy 1 (Task Migration):} If accuracy is insufficient and a nearby vehicle is available, the training task is migrated. The cost reflects latency $\tau_v^{\text{mig}}$ and energy $e_v^{\text{mig}}$:
\begin{align*}
\text{Cost}_1 = \alpha \cdot \tau_v^{\text{mig}} + \beta \cdot e_v^{\text{mig}}.
\end{align*}

\noindent\textbf{Strategy 2 (Abandonment):} If migration is infeasible and accuracy is unsatisfactory, training is aborted. The cost accounts for wasted energy $\hat{e}_v^{t,m}$ and missed contribution:
\begin{align*}
\text{Cost}_2 = \beta \cdot \hat{e}_v^{t,m} + \gamma \cdot q_v^{*,t}.
\end{align*}

Let $z_v^{t,m} \in \{0,1,2\}$ denote the chosen strategy. The scheduling objective minimizes total expected cost:
\begin{align*}
\min_{z_v^{t,m}} \sum_{v \in \mathcal{V}} \sum_{t \in \mathcal{T}} \sum_{m=1}^{M} \text{Cost}_{z_v^{t,m}} \cdot \mathbb{I}[z_v^{t,m}].
\end{align*}

At runtime, all three strategies are evaluated, and the one with the lowest cost is selected, ensuring resilient training under dynamic vehicle mobility.

\section{Performance Evaluation}

\subsection{Experimental Setup}

We establish a simulator tailored to hierarchical IoV federated fine-tuning. Each RSU is assigned a specific task and coordinates localized fine-tuning on vehicles with heterogeneous system capacities.
RSUs are placed at traffic hotpots with constrained spatial coverage, and vehicles follow actual trajectories, dynamically entering or exiting RSU zones. It incorporates intermittent connectivity, early departures, and task handoffs, reflecting real-world deployment challenges.
In our experiments, each fine-tuning task spans 400 global communication rounds, with active vehicles performing five local update steps per round using the Adam optimizer with learning rate $1 \times 10^{-5}$ and batch size 10. Vehicles receive unequal, randomly sampled portions of task-specific datasets with non-i.i.d. distributions. To emulate realistic mobility, T-Drive dataset~\cite{zheng2011t-drive} containing GPS traces is used. 
For the proposed UCB-DUAL algorithm, we set  $\alpha=0.5$ and $\gamma=2$ to balance latency and accuracy. The total energy budget \textbf{$E_{\text{total}}$} is a configurable, hardware-interpretable threshold derived from the physical power and battery discharge constraints. We set the dual learning rate $\omega = 0.05$ and the exploration factor $\epsilon = \sqrt{2}$. To ensure stable estimation of the reward-cost statistics, a warm-up period of $Q = 6$ rounds is employed.

\noindent\textbf{Models, Tasks, and Datasets.} 
To ensure robustness across diverse IoV scenarios, we employ ViT-Base~\cite{dosovitskiy2020image} and Swin-Base~\cite{liu2021swin} as backbones, with LoRA adapters integrated into attention and feed-forward linear layers following~\cite{hu2021lora}.
Our experiments encompass three representative perception tasks characterized by high-dimensional data and inherent complexity: 1) \textbf{Multi-scale Urban Object Detection (OD):} on the Road-Traffic dataset~\cite{road-traffic} for high-precision multi-object sensing in dense traffic; 2) \textbf{Scene-level Semantic Segmentation (SS)} on Cityscapes~\cite{Cordts2016Cityscapes} for fine-grained lane and road structure parsing; and 3) \textbf{Context-aware Traffic Sign and Road-symbol Classification (TC)} on TSRD~\cite{evandu2025traffic} for robust sign identification under non-iid vehicular data distributions.
This diverse multi-task configuration, characterized by heterogeneous data formats and task complexities, provides a comprehensive foundation for evaluating our federated resource scheduling and fine-tuning strategies.

\noindent\textbf{Baseline Methods.}
We compare the proposed approach with the following representative baselines:

\begin{itemize}
    \item \textbf{HomoLoRA~\cite{mcmahan2017communication}:} All vehicles adopt a fixed LoRA rank, and model updates are aggregated using FedAvg.
    \item \textbf{HetLoRA~\cite{cho2024heterogeneous}:} Vehicles adopt heterogeneous LoRA ranks based on local capabilities, utilizing zero-padding and gradient-based self-pruning to reduce redundancy.
    \item \textbf{FedRA~\cite{su2024fedra}:} A federated tuning approach that randomly allocates subsets of model layers to heterogeneous vehicles and aggregates adapter updates accordingly.
\end{itemize}

\noindent\textbf{Evaluation Metrics.}
We evaluate system performance using the following metrics:

\begin{itemize}
    \item \textbf{Reward.} A unified metric balancing accuracy and latency, defined by the weighted objective in Equation~(\ref{eq:multi_task_opt}). Higher values indicate better efficiency and performance.

    \item \textbf{Avg. Accuracy.} The mean of the best accuracy per task, reflecting generalization across heterogeneous tasks.

    \item \textbf{Latency.} Average latency per federated fine-tuning communication round.

    \item \textbf{Energy Consumption.} Average energy consumption per communication round, measuring resource efficiency under mobility and communication overhead.
\end{itemize}

\subsection{Evaluation Results}

\begin{table}[t]
\centering
\caption{Comparison with baselines varying models. Best results are \textbf{bolded}, and second best are \underline{underlined}. Comm. denotes the average parameter upload volume per round.}
\label{tab:combined_results}
\resizebox{\columnwidth}{!}{
\begin{tabular}{lcccccc}
\toprule
\textbf{Method} & \textbf{Model} & \textbf{Reward} $\uparrow$ & \textbf{Avg. Acc.}(\%) $\uparrow$ & \textbf{Lat.}(s) $\downarrow$ & \textbf{Energy($J$)} $\downarrow$ & \textbf{Comm.}(M) $\downarrow$ \\
\midrule
\multirow{2}{*}{HomoLoRA}
    & ViT  & $367.7_{\pm 0.1}$ & $82.5_{\pm 0.4}$ & $82.9_{\pm 0.2}$ & $3{,}851.1_{\pm 0.5}$ & 29.8 \\
    & Swin & $345.3_{\pm 0.2}$ & $80.9_{\pm 0.6}$ & $90.5_{\pm 0.5}$ & $4{,}198.7_{\pm 0.8}$ & 85.8 \\
\cmidrule{1-7}
\multirow{2}{*}{HetLoRA}
    & ViT  & $390.7_{\pm 0.1}$ & $\underline{83.9_{\pm 0.3}}$ & $71.2_{\pm 0.3}$ & $3{,}601.9_{\pm 0.4}$ & 45.1 \\
    & Swin & $355.8_{\pm 0.1}$ & $83.9_{\pm 0.4}$ & $96.1_{\pm 0.3}$ & $3{,}993.6_{\pm 0.6}$ & 129.9 \\
\cmidrule{1-7}
\multirow{2}{*}{FedRA}
    & ViT  & $\underline{410.1_{\pm 0.2}}$ & $83.8_{\pm 0.5}$ & $\underline{60.4_{\pm 0.2}}$ & $\underline{3{,}186.1_{\pm 0.6}}$ & \underline{27.8} \\
    & Swin & $\underline{373.4_{\pm 0.2}}$ & $\underline{84.4_{\pm 0.5}}$ & $\underline{82.8_{\pm 0.4}}$ & $\underline{3{,}931.9_{\pm 0.7}}$ & \underline{67.1} \\
\midrule
\multirow{2}{*}{\textbf{Ours}}
    & ViT  & $\mathbf{431.2_{\pm 0.1}}$ & $\mathbf{85.2_{\pm 0.3}}$ & $\mathbf{50.7_{\pm 0.2}}$ & $\mathbf{2{,}561.5_{\pm 0.4}}$ & $\mathbf{26.1}$ \\
    & Swin & $\mathbf{407.0_{\pm 0.1}}$ & $\mathbf{85.9_{\pm 0.2}}$ & $\mathbf{67.8_{\pm 0.2}}$ & $\mathbf{3{,}062.4_{\pm 0.5}}$ & $\mathbf{65.8}$ \\
\bottomrule
\end{tabular}
}
\end{table}

\begin{table}[t]
  \centering
  \caption{Comparison with baselines varying Perception Tasks. Best results are \textbf{bold}, and second best are \underline{underlined}.}
  \label{tab:task_reward_comparison}
  \resizebox{\columnwidth}{!}{
  \begin{tabular}{lcccc}
    \toprule
    \textbf{Method} & \textbf{Model} & \textbf{OD Reward} $\uparrow$ & \textbf{SS Reward} $\uparrow$ & \textbf{TC Reward} $\uparrow$ \\
    \midrule
    \multirow{2}{*}{HomoLoRA}
        & ViT  & $132.2_{\pm 0.2}$ & $107.1_{\pm 0.1}$ & $142.5_{\pm 0.2}$ \\
        & Swin & $109.2_{\pm 0.1}$ & $\underline{115.6_{\pm 0.1}}$ & $126.2_{\pm 0.1}$ \\
    \cmidrule{1-5}
    \multirow{2}{*}{HetLoRA}
        & ViT  & $129.5_{\pm 0.1}$ & $124.0_{\pm 0.1}$ & $147.7_{\pm 0.1}$ \\
        & Swin & $112.7_{\pm 0.1}$ & $106.2_{\pm 0.2}$ & $138.7_{\pm 0.1}$ \\
    \cmidrule{1-5}
    \multirow{2}{*}{FedRA}
        & ViT  & $\underline{132.6_{\pm 0.2}}$ & $\underline{127.9_{\pm 0.1}}$ & $\underline{156.4_{\pm 0.2}}$ \\
        & Swin & $\underline{117.6_{\pm 0.1}}$ & $113.4_{\pm 0.6}$ & $\underline{146.9_{\pm 0.1}}$ \\
    \cmidrule{1-5}
    \multirow{2}{*}{\textbf{Ours}}
        & ViT  & $\mathbf{136.5_{\pm 0.1}}$ & $\mathbf{140.2_{\pm 0.1}}$ & $\mathbf{157.3_{\pm 0.1}}$ \\
        & Swin & $\mathbf{118.1_{\pm 0.1}}$ & $\mathbf{148.2_{\pm 0.4}}$ & $\mathbf{148.1_{\pm 0.1}}$ \\
    \bottomrule
  \end{tabular}
  }
\end{table}

\begin{figure}[t]
  \centering
  \begin{minipage}{0.49\linewidth}
    \centering
    \includegraphics[width=\linewidth]{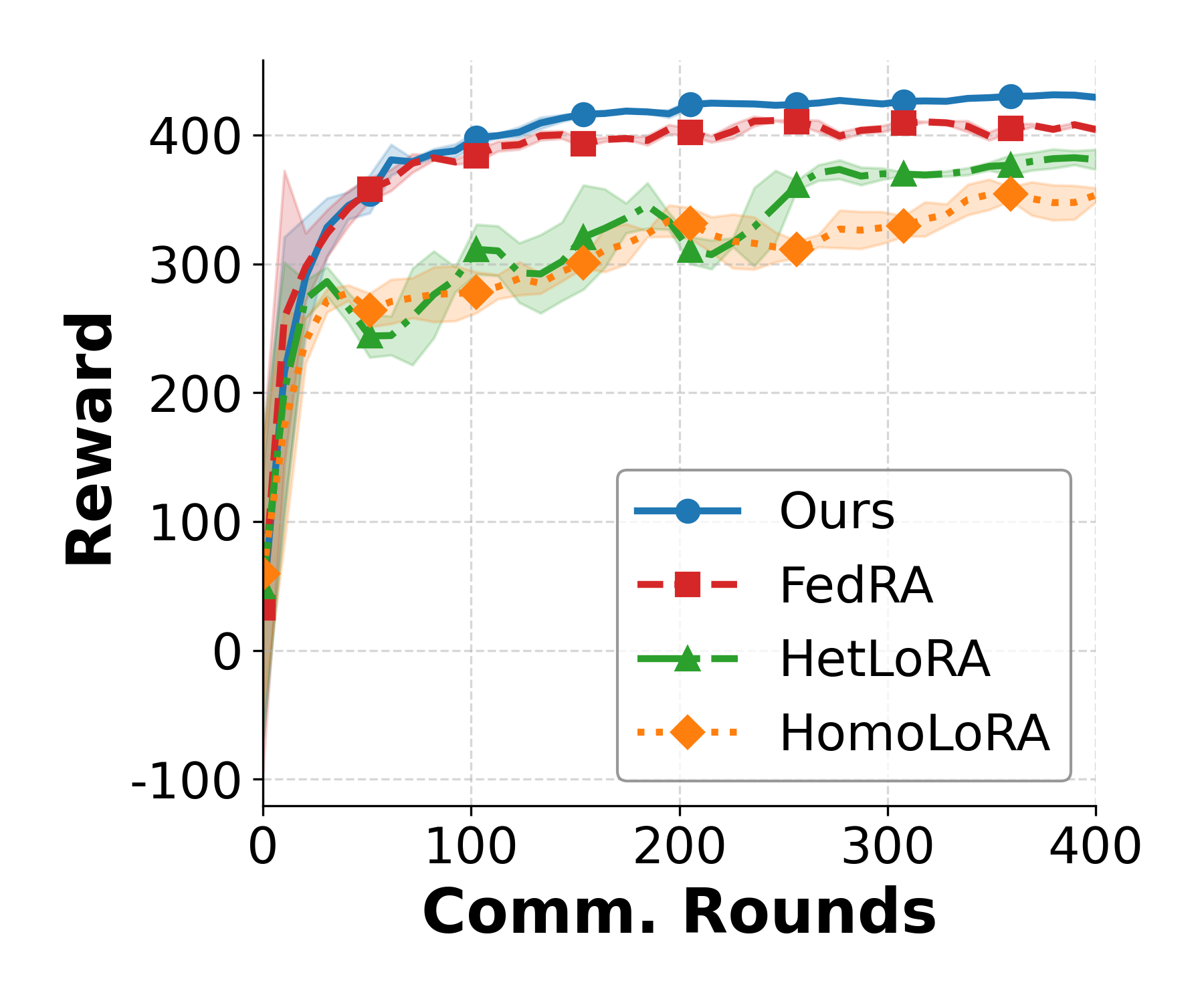}
    \caption{Reward over communication rounds.}
    \label{fig:reward}
  \end{minipage}
  \hfill
  \begin{minipage}{0.49\linewidth}
    \centering
    \includegraphics[width=\linewidth]{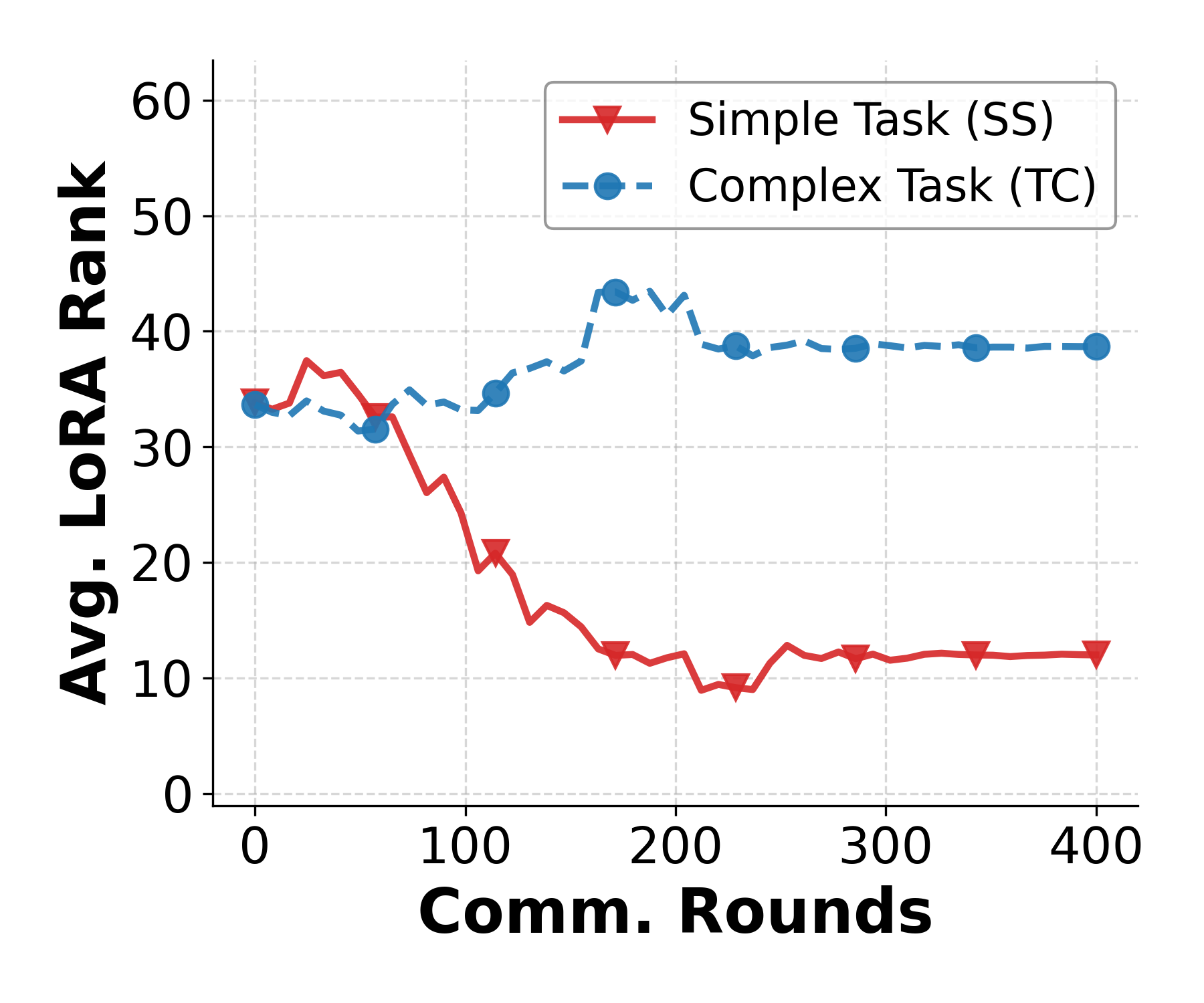}
    \caption{LoRA rank evolution across tasks.}
    \label{fig:rank_evo}
  \end{minipage}
\end{figure}

\begin{figure}
    \centering
    \includegraphics[width=1\linewidth]{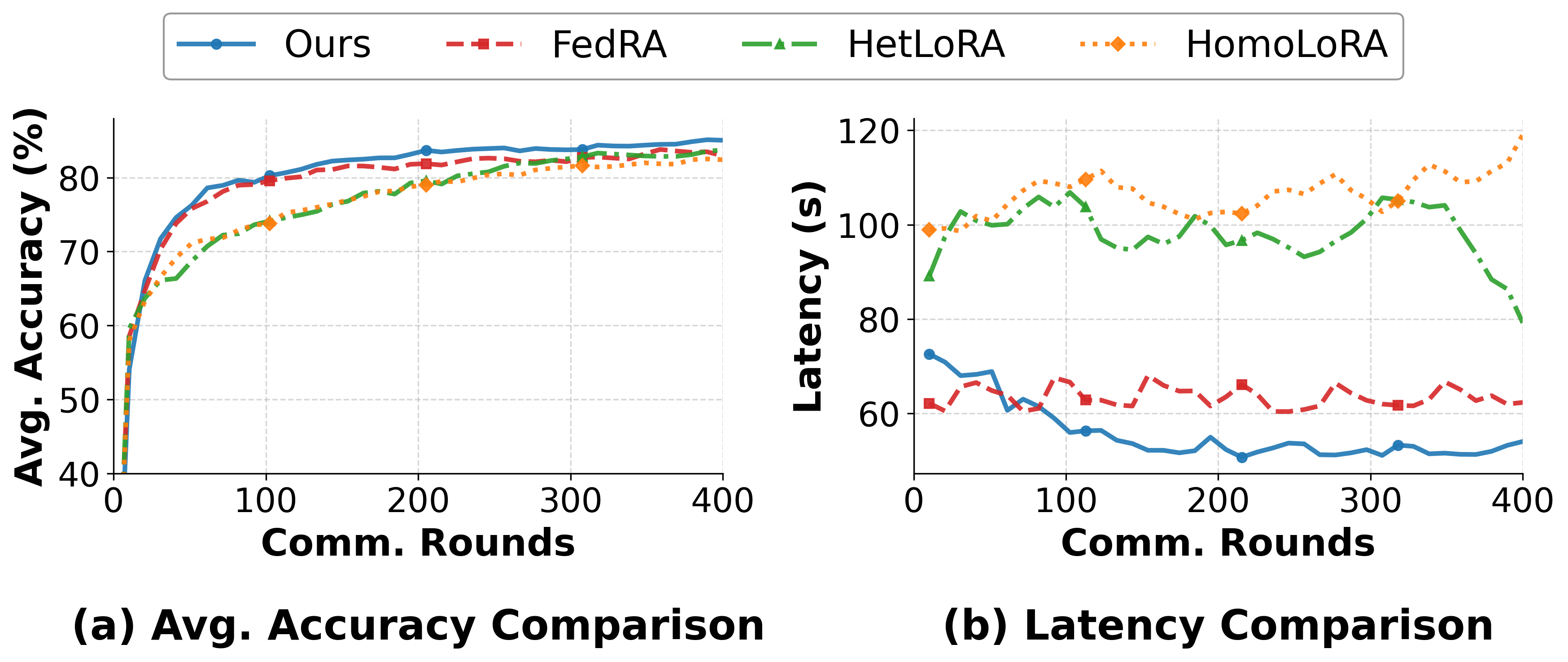}
    \caption{Accuracy and latency comparison across methods.}
    \label{fig:acc-latency}
\end{figure}


As shown in Table~\ref{tab:combined_results}, our method consistently achieves the best overall trade-off across all metrics. Compared to HomoLoRA, a fixed-rank scheme, our approach significantly improves accuracy while reducing latency and energy consumption. While HetLoRA employs heterogeneous ranks, its reliance on padding and pruning introduces additional overhead, leading to suboptimal efficiency. FedRA achieves competitive performance through randomized allocation but lacks energy-aware coordination for dynamic workloads. In contrast, our framework jointly optimizes task-level energy and vehicle-level ranks, yielding superior reward and system efficiency. Importantly, the RSU-side SVD computation accounts for only \textbf{2\%} of the total round latency, being significantly outweighed by the local fine-tuning and transmission phases, which confirms its minimal impact on overall system performance.

Figure~\ref{fig:reward} shows reward evolution using the ViT backbone. Our method outperforms all baselines in convergence speed and steady-state reward, confirming its superior exploration and resource efficiency under time-varying conditions. Figure~\ref{fig:acc-latency} further contrasts accuracy and latency dynamics. While HomoLoRA suffers from limited accuracy due to inflexible rank assignments and HetLoRA faces latency instability from heterogeneous aggregation, our approach maintains stable accuracy growth with consistently low latency. Finally, the performance gap between our method and FedRA highlights that random layer allocation, while reducing overhead, fails to capture the optimal adaptation required for complex tasks.

Table~\ref{tab:task_reward_comparison} summarizes peak rewards across perception tasks. Our method consistently outperforms all baselines by jointly optimizing inter-task energy budgets and intra-task rank adaptation, enabling dynamic resource allocation based on task complexity and runtime conditions. Unlike baselines that rely on independent, static, or randomized strategies, our approach effectively exploits cross-task coordination under global energy constraints.

\begin{figure}[t]
  \centering
  \begin{minipage}{0.49\linewidth}
    \centering
    \includegraphics[width=\linewidth]{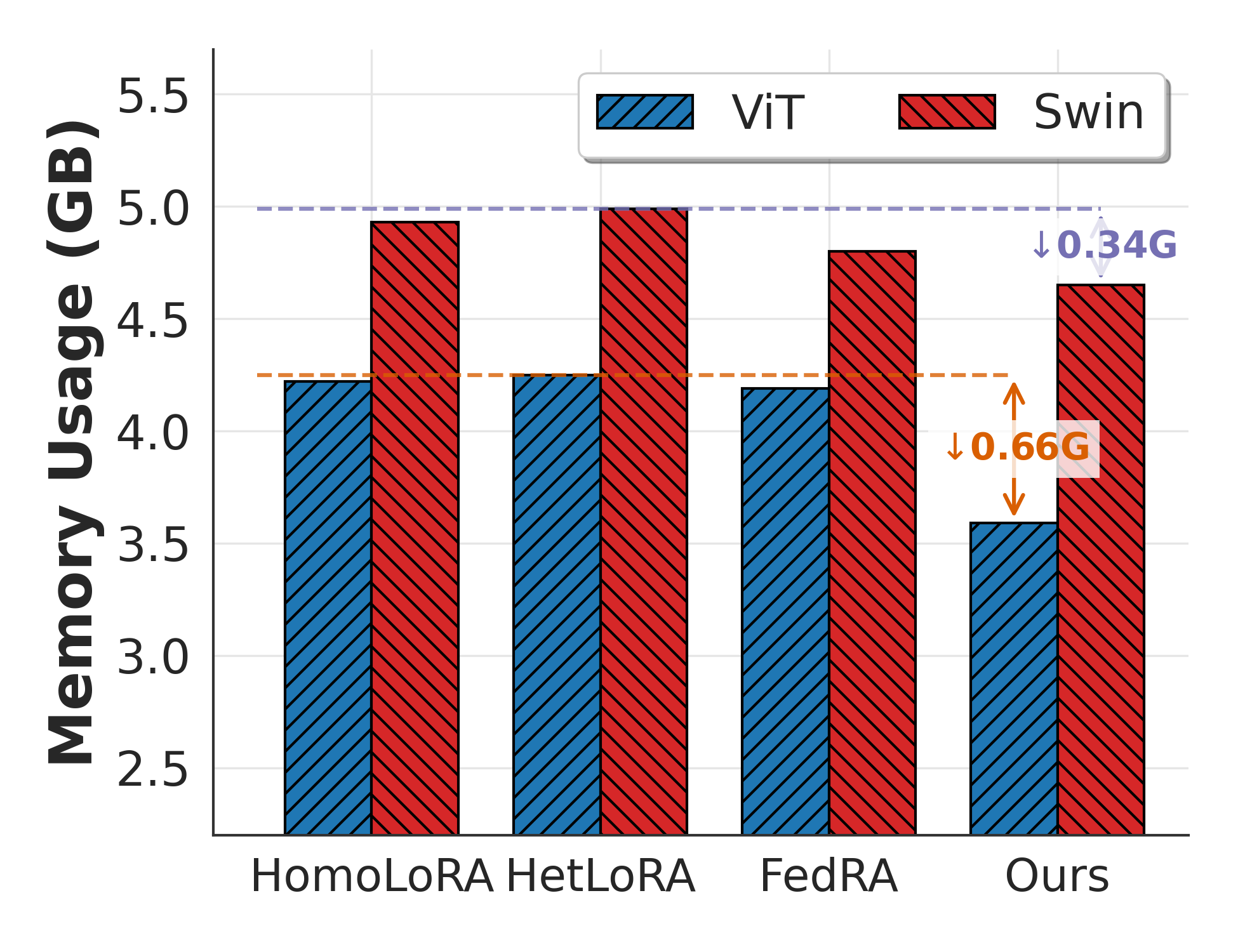}
    \caption{Comparison of CUDA memory usage.}
    \label{fig:mem_comp}
  \end{minipage}
  \hfill
  \begin{minipage}{0.49\linewidth}
    \centering
    \includegraphics[width=\linewidth]{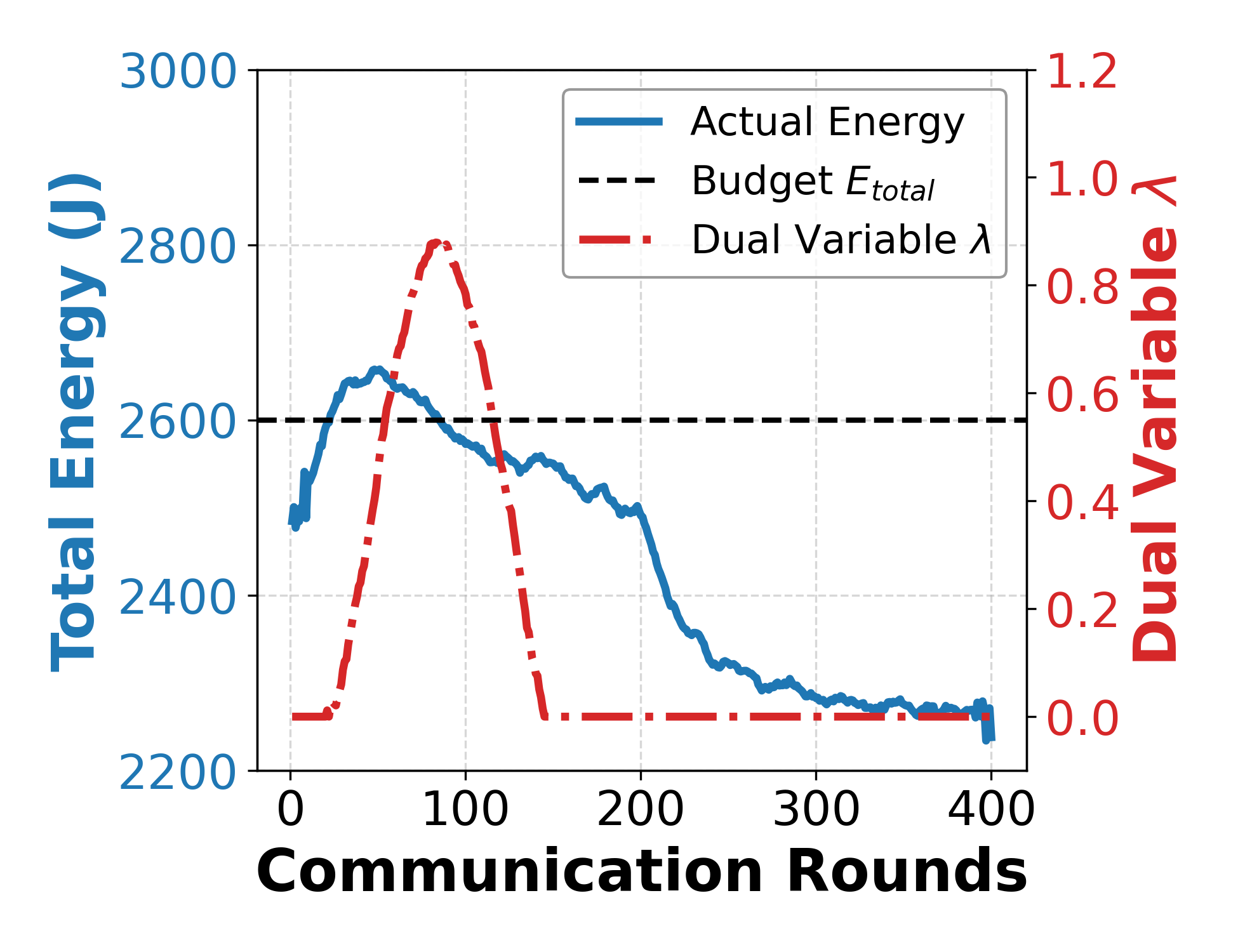}
    \caption{Energy and dual variable evolution.}
    \label{fig:energy_dual}
  \end{minipage}
\end{figure}

Fig.~\ref{fig:mem_comp} evaluates CUDA memory consumption on the NVIDIA RTX 4090. Our method consistently achieves the lowest footprint across both backbones, demonstrating superior resource efficiency. In contrast, HomoLoRA’s fixed-rank approach incurs unnecessary overhead by ignoring device heterogeneity, while HetLoRA’s reliance on zero-padding for adapter alignment introduces structural redundancy. Although FedRA offers partial memory reduction via randomized layer assignment, its stochastic nature lacks energy-aware adaptation. Conversely, our energy-aware SVD rank construction enables highly fine-grained parameter reduction while simultaneously preserving adaptation quality, yielding the most substantial and stable memory savings.

\subsection{Ablation Study}

\begin{table}[t]
\centering
\caption{Ablation results on energy-aware scheduler and mobility-aware strategy.}
\label{tab:ablation}
\resizebox{\columnwidth}{!}{
\begin{tabular}{lccccc}
\toprule
\textbf{Variant} & \textbf{Model} & \textbf{Reward} $\uparrow$ & \textbf{Avg. Acc.}(\%) $\uparrow$ & \textbf{Lat.}(s) $\downarrow$ & \textbf{Energy($J$)} $\downarrow$ \\
\midrule
\multirow{2}{*}{\textbf{Ours (Full)}}
    & ViT  & $\mathbf{431.2_{\pm 0.1}}$ & $\mathbf{85.2_{\pm 0.3}}$ & $\mathbf{50.7_{\pm 0.2}}$ & $\mathbf{2{,}561.5_{\pm 0.4}}$ \\
    & Swin & $\mathbf{407.0_{\pm 0.1}}$ & $\mathbf{85.9_{\pm 0.2}}$ & $\mathbf{67.8_{\pm 0.2}}$ & $\mathbf{3{,}062.4_{\pm 0.5}}$ \\
\cmidrule{1-6}
\multirow{2}{*}{\shortstack[l]{w/o Mobility-aware\\Scheduling}} 
    & ViT & $412.8_{\pm 0.1}$ & $84.2_{\pm 0.1}$ & $62.5{\pm 0.4}$ & $3,150.2{\pm 0.5}$ \\
 & Swin & $382.4_{\pm 0.2}$ & $84.8_{\pm 0.3}$ & $81.2{\pm 0.2}$ & $3,680.5{\pm 0.2}$ \\
\cmidrule{1-6}
\multirow{2}{*}{\shortstack[l]{w/o Energy-aware\\Scheduler}}
    & ViT  & $402.5_{\pm 0.1}$ & $83.9_{\pm 0.4}$ & $65.8_{\pm 0.3}$ & $3{,}422.4_{\pm 0.5}$\\
    & Swin & $370.4_{\pm 0.1}$ & $84.3_{\pm 0.3}$ & $88.5_{\pm 0.4}$ & $3{,}965.2_{\pm 0.7}$ \\
\bottomrule
\end{tabular}
}
\end{table}

Table~\ref{tab:ablation} evaluates the contributions of the energy-aware rank scheduler and mobility-aware strategy across ViT and Swin backbones. Removing the energy-aware scheduler consistently degrades performance, manifesting in lower cumulative rewards and accuracy alongside increased latency and communication overhead. Our full framework overcomes these bottlenecks by jointly coordinating task priorities and rank allocations. Fig.~\ref{fig:rank_evo} corroborates this: the scheduler dynamically lowers ranks for simple tasks (e.g., SS) to prevent over-provisioning and increases them for complex tasks (e.g., TC) to preserve model capacity.
Disabling mobility-aware scheduling further impairs efficiency by failing to anticipate vehicle departures, resulting in wasted computation from discarded local updates. Although the mobility-aware strategy itself is lightweight, this strategy enhances system robustness by preserving intermediate results and minimizing task completion delays. Together, these results confirm that the synergy between fine-grained rank adaptation and mobility-aware scheduling is essential for achieving robust system optimization in dynamic vehicular networks.

\subsection{Constraint Enforcement and Dual Variable Dynamics}

Fig.~\ref{fig:energy_dual} illustrates the constraint-enforcement of UCB-DUAL by tracking total energy and the dual variable $\lambda$. Initially, $\lambda$ rises as energy consumption nears or exceeds the budget, penalizing high-rank configurations to restore system feasibility. As performance converges, energy stabilizes below the limit, causing  $\lambda$ to decay. Notably, after approximately 150 rounds, the system autonomously shifts to lower-rank configurations without sacrificing accuracy, ensuring the hard energy constraint is satisfied while maximizing efficiency. This confirms that UCB-DUAL effectively enforces global budgets while adaptively balancing the inherent performance–energy trade-off within dynamic IoV networks.

\subsection{Scalability Analysis}
\begin{figure}[t]
  \centering
  \begin{minipage}{0.49\linewidth}
    \centering
    \includegraphics[width=\linewidth]{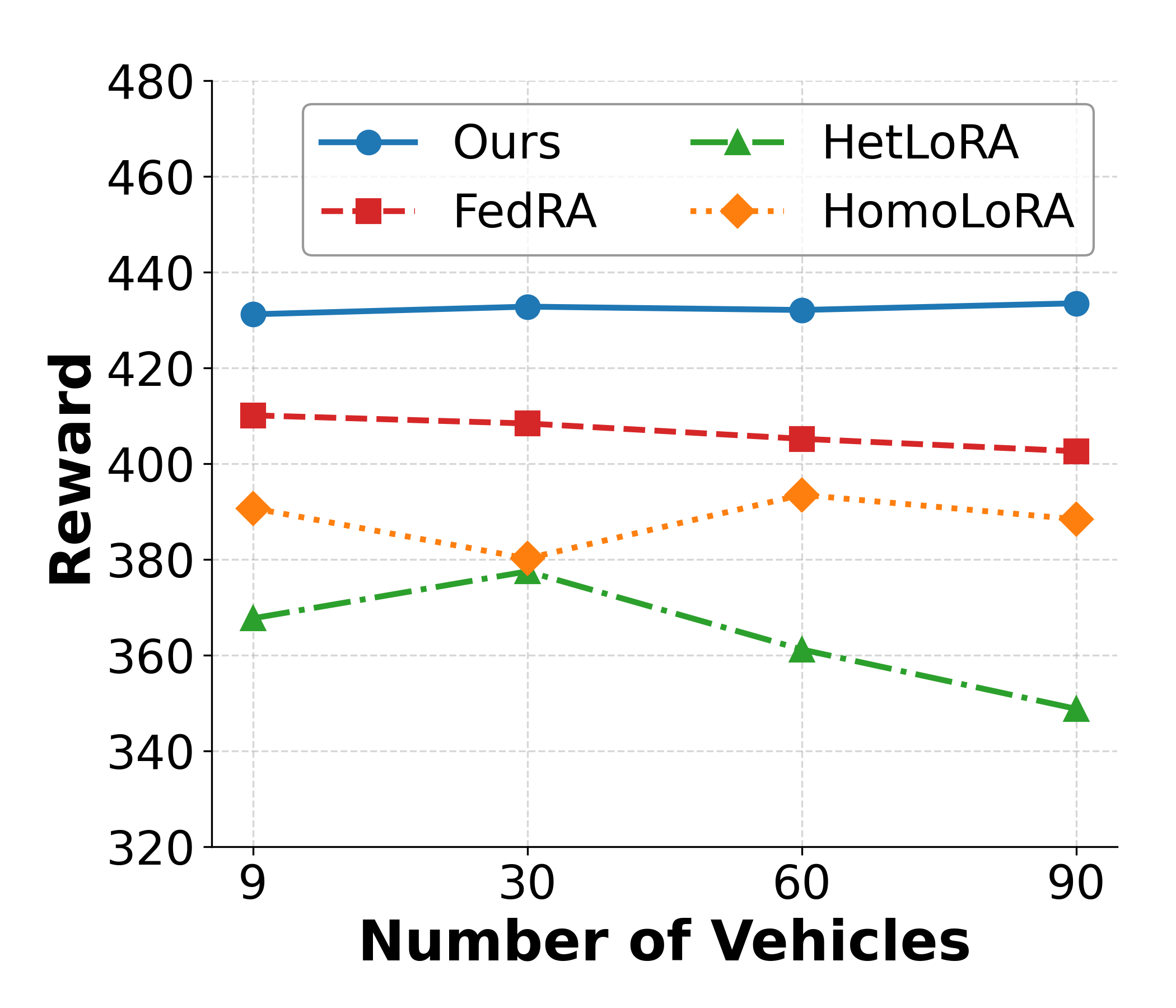}
    \caption{Scalability under increasing number of vehicles.}
    \label{fig:agent_scalability}
  \end{minipage}
  \hfill
  \begin{minipage}{0.49\linewidth}
    \centering
    \includegraphics[width=\linewidth]{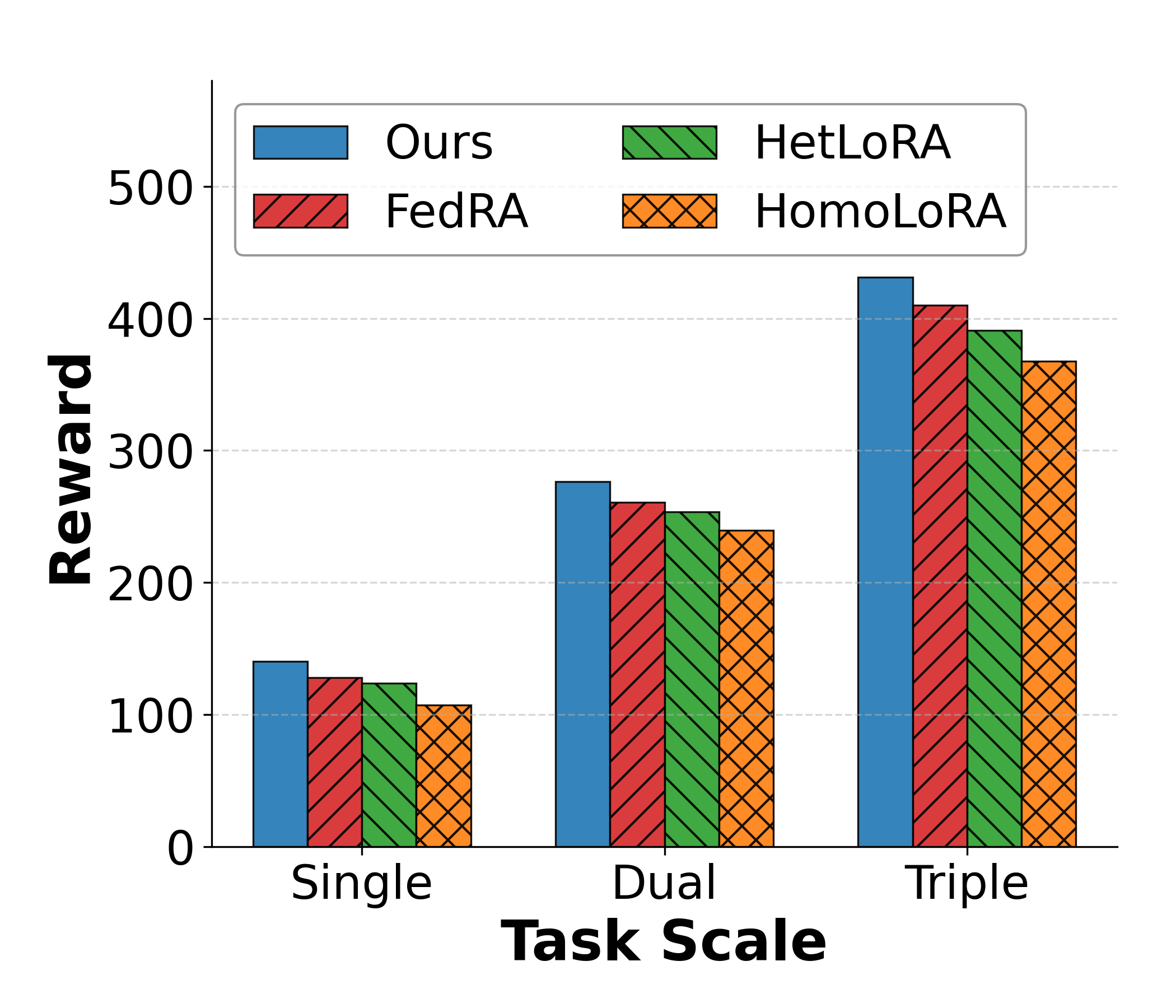}
    \caption{Performance under varying task scales.}
    \label{fig:task_scalability}
  \end{minipage}
\end{figure}
Using ViT as the backbone, we first evaluate scalability with respect to the number of participating vehicles by varying the fleet size from 9 to 90 while keeping the environment unchanged. As shown in Fig.~\ref{fig:agent_scalability}, our method maintains consistently high and stable cumulative reward as vehicle density increases, demonstrating strong robustness under dense participation. In contrast, FedRA shows a gradual performance degradation as contention intensifies, reflecting the limitations of random layer allocation. HetLoRA exhibits noticeable fluctuations under heterogeneous participation, while HomoLoRA degrades steadily due to its rigid, uniform rank configuration. These results indicate that adaptive scheduling with rank allocation enables effective coordination at scale.

We further examine task scalability by increasing the number of concurrent tasks from one to three, as shown in Fig.~\ref{fig:task_scalability}. Our method consistently delivers the highest cumulative reward, demonstrating superior robustness as the task load grows. While FedRA outperforms HomoLoRA and HetLoRA by supporting heterogeneous participation, its performance gains saturate due to the lack of task-aware adaptation. In contrast, fixed-rank and padding-based strategies suffer from increasing inefficiency as task diversity grows. By explicitly modeling task difficulty and jointly optimizing task-level energy budgets with vehicle-level rank adaptation, our approach sustains stable reward growth and superior scalability.

\section{CONCLUSION}
\label{sec:conclusion}

This paper presents an energy-aware federated fine-tuning framework for dynamic, resource-constrained IoV environments. By combining LoRA-based multi-task adaptation with adaptive rank scheduling, our approach jointly addresses vehicle heterogeneity, task diversity, mobility, and energy constraints. We formulate rank selection as a constrained online learning problem and propose a lightweight UCB-DUAL mechanism for efficient rank adaptation with minimal coordination overhead. Extensive evaluations on a large-scale, trajectory-driven simulator show consistent improvements over strong baselines in reward, accuracy, latency, energy, communication, and memory efficiency, while ablation and scalability studies confirm the importance of adaptive scheduling as task and vehicle scales grow.

\ifCLASSOPTIONcaptionsoff
  \newpage
\fi

\bibliographystyle{IEEEtran}
\bibliography{main}
\newpage

\end{document}